\documentclass[letterpaper, 10 pt, conference]{ieeeconf}  %

\IEEEoverridecommandlockouts                      %
\usepackage{style_cdc}
\usepackage{amsmath}  
\usepackage{amssymb}
\usepackage{xcolor}
\usepackage{float}
\usepackage{caption} 
\usepackage{url}
\usepackage{hyperref}
\usepackage{cite}
\usepackage{bm}
\usepackage{graphicx}
\usepackage[ruled,vlined,noline,linesnumbered]{algorithm2e}
\usepackage{subcaption}
\usepackage{wrapfig}
\usepackage{fancyhdr}
\newcommand{\x}{{\bm x}}
\newcommand{\bu}{{\bm u}}

\title{\LARGE \bf
Towards Bias Correction of FedAvg over Nonuniform and Time-Varying Communications
}

\author{Ming Xiang, Stratis Ioannidis, Edmund Yeh, Carlee Joe-Wong, and Lili Su%
\thanks{M. Xiang, S. Ioannidis, E. Yeh, and L. Su are with Department of Electrical and Computer Engineering, Northeastern University,
    Boston, MA 02215, USA. 
    C. Joe-Wong is with Department of Electrical and Computer Engineering, Carnegie Mellon University, Pittsburgh, PA 15213, USA. 
\thanks{There is a typo in Lemma \ref{lemma: local step perturbation} of the short version, we have corrected it in this full version.}  
}
}

\begin{document}

\maketitle
\thispagestyle{empty}
\pagestyle{empty}

\begin{abstract}
Federated learning (FL) is a decentralized learning framework wherein a parameter server (PS) and a collection of clients collaboratively train a model via minimizing a global objective. %
Communication bandwidth is a scarce resource; in each round, the PS aggregates the updates from a subset of clients only.  
In this paper, we focus on non-convex minimization that is vulnerable to non-uniform and  time-varying communication failures between the PS and the clients. 
Specifically, in each round $t$, the link between the PS and client $i$ is active with probability $p_i^t$, which is {\em unknown} to both the PS and the clients. 
This arises when %
the channel conditions are heterogeneous across clients and are changing over time.  

We show that when the $p_i^t$'s are not uniform, %
{\em Federated Average} (FedAvg) -- the most widely adopted FL algorithm -- fails to minimize the global objective. Observing this, we propose {\em Federated Postponed Broadcast} (FedPBC) which is a simple variant of FedAvg. It differs from FedAvg in that the PS postpones broadcasting the global model till the end of each round.  
We show that FedPBC converges to a stationary point of the original objective. %
The introduced staleness is mild and there is no noticeable slowdown. Both theoretical analysis and numerical results are provided. 
On the technical front, %
postponing the global model broadcasts enables implicit gossiping among the clients with active links at round $t$. 
Despite $p_i^t$'s are time-varying, %
we are able to bound the perturbation of the global model dynamics %
via the techniques of controlling the gossip-type information mixing errors. %
\end{abstract}
\section{INTRODUCTION}
\label{sec: intro}
Federated learning (FL) %
is a distributed learning paradigm wherein a parameter server (PS) and a large collection of clients collaboratively learn a machine learning model with clients' local data undisclosed \cite{mcmahan2017communication,kairouz2021advances} to the PS. 
The global objetives are often non-convex. Communication bandwidth is a scarce resource. %
In each round, the PS aggregates the updates from a subset of clients only -- either proactively \cite{mcmahan2017communication,kairouz2021advances} or passively \cite{Li2020,philippenko2020bidirectional,wang2022}.  
\begin{figure} 
    \centering
    \includegraphics[width=0.9\linewidth,trim={5cm 6cm 5cm 4cm},clip]{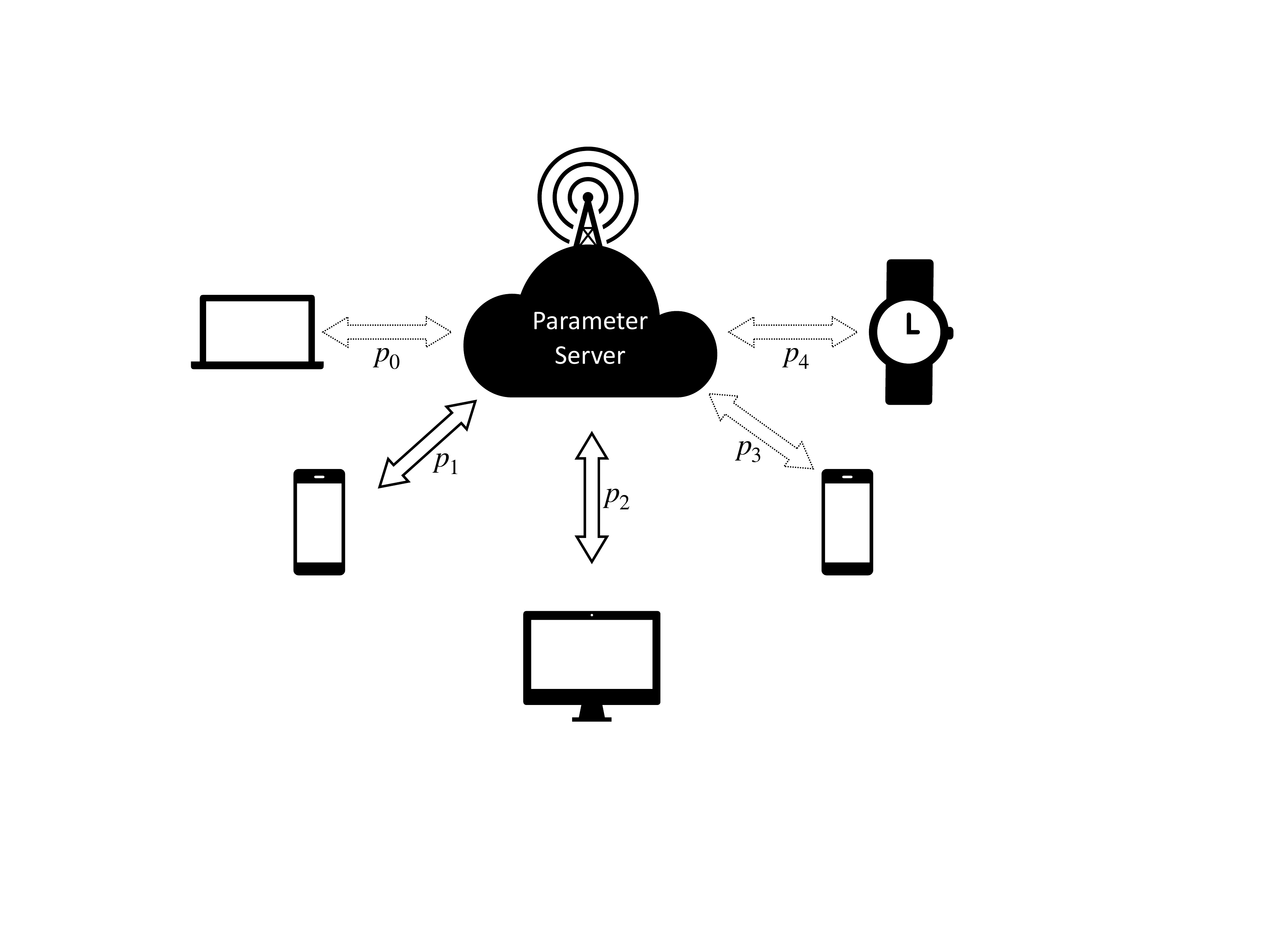}
    \caption{A federated learning system with heterogeneous devices: Solid arrows indicate active links and dashed arrows are inactive links.} 
    \label{fig:FL systems}
\end{figure}
\setlength{\parskip}{0mm}
A FL system is often deployed in a uncontrolled environment, wherein the channel conditions between the PS and the clients could be highly heterogeneous and time-varying \cite{mcmahan2017communication}. 
To capture this, in this paper, we consider non-convex minimization that is vulnerable to non-uniform and time-varying link failures between the PS and the clients. 
Specifically, in each round, the link between the PS and client $i$ is active with probability $p_i^t$, which %
is unknown to both the PS and the clients. A generic FL system of interest is illustrated in Fig.\,\ref{fig:FL systems}. To the best of our knowledge, the convergence of FL in the presence of non-uniform and time-varying communication is overall under-explored.

Our setup can be viewed as a special case of the general client unavailability, has received intensive attention recently~\cite{kairouz2021advances}. %
Nevertheless, existing methods are not applicable to our problem.  
In the seminal works \cite{mcmahan2017communication,Li2020}, the PS chooses $K$ clients either uniformly at random or proportionally to clients' local data volume. %
Neither of theses client selection methods is feasible when $p_i^t$'s are unknown and time-varying. 
In \cite{Li2020,philippenko2020bidirectional,kairouz2021advances,pmlr-v180-jhunjhunwala22a}, the PS waits for the $K$ fastest responses. The correctness of their algorithms crucially relies  on the fact that the response probability of each client is known.  
Ruan et al.~\cite{ruan2021towards} considered a generalized random client unavailability, yet required the response probability to be fixed. 
Time-varying response rates are also considered in \cite{wang2022, perazzone2022communication,gu2021fast}. 
For the methods in \cite{wang2022} to converge to stationary points, the response rates need to be ``balanced'' in the sense that 
either (1) the $p_i^t$'s are deterministic and satisfy the regularized participation, i.e., $\sum_{\tau=1}^P p_i^{t_0+\tau} =  \mu$ for all clients at all $t_0\in \{0, P, 2P, \cdots\}$ where $P$ is some carefully chosen integer; or (2) $p_i^t$'s are random and satisfy $\expect{p_i^t} = \mu$ for all clients and sufficiently many $t$.  In contrast, we do not require such rate ``balanceness''. 
Perazzone et al.\,\cite{perazzone2022communication} analyzed the convergence of {FedAvg} under time-varying client participation rates. Nevertheless, they assumed (1) a uniform participation rate in each round, i.e., $p_i^t = p_j^t$ for any pair of clients, and 
 (2) bounded stochastic gradient. 
Gu et al.\,\cite{gu2021fast} considered general client unavailability patterns for both strongly convex and non-convex global objectives. For non-convex objectives (which is our focus), they required that the consecutive unavailability rounds of a client to be deterministically upper bounded, which does not hold even for the simple uniform and time-invariant response rates. Moreover, they required the noise of the stochastic gradient to be uniformly upper bounded with probability 1.

\vskip 0.6\baselineskip 
\noindent{\bf Contributions.} 
Our contributions is three-fold: 
\begin{itemize}
\item We identify simple instances and show both analytically and numerically that when the $p_i^t$'s are not uniform {\em Federated Average} (FedAvg) -- the most widely adopted FL algorithm -- fails to minimize the global objective.
\item We propose {\em Federated Postponed Broadcast} (FedPBC). %
It differs from FedAvg in that the PS postpones broadcasting the global model till the end of each round. 
We show in Theorem 1 that, in expectation, FedPBC converges to a stationary point of the global objective. %
The correctness of our FedPBC 
{\em neither} impose any ``balancedness'' requirement on $p_i^t$'s {\em nor} require the stochastic gradients or their noises to be bounded. Moreover, compared with \cite{gu2021fast,wang2022}, FedPBC works under a much relaxed bounded-dissimilarity assumption.

On the technical front, %
postponing the global model broadcasts enables implicit gossiping among the clients with active links.  
Hence, we mitigate the perturbation %
caused by non-uniform and time-varying $p_i^t$ via the techniques of controlling information mixing errors. %

\item We validate our results empirically both on the counterexample and by using Synthetic $(1,1)$ dataset \cite{li2020federated}.
    The numerical results in the former show that FedPBC successfully corrects the bias when $p_i^t$'s are static but non-uniform (i.e., $p_i^t = p_i$) %
    while FedAvg does not. Moreover, the staleness is mild and there is no noticeable slowdown.
    In the latter, we further investigate {\it time-varying} link activation rates such that they first satisfy $\expect{p_i^t}\triangleq \prob{Z_i^t = i},$ 
    where $Z_i^t$ follows Zipf distribution and is \iid over the time horizon. 
    It is then clipped to ensure a lower bound.
    The results show FedPBC significantly outperforms FedAvg.
\end{itemize}

\section{Problem Formulation}
A FL system consists of %
one central PS and $m$ clients %
that collaboratively  minimize
\begin{align}
\min\limits_{\x\in\reals^d}F\pth{\x} = \frac{1}{m} \sum_{i\in[m]} F_i \pth{\x}, 
\label{eq: global objective}
\end{align}
where
$F_i\pth{\x} = \Expect_{\xi_i\in \calD_i}[\ell_i \pth{\x;\xi_i}]$ is the local objective, 
$\calD_i$ is the local distribution, %
$\xi_i$ is a stochastic sample that client $i$ has access to,
and $\ell_i$ is the local loss function.
The loss function can be non-convex. 
We are interested in solving Eq.\,\eqref{eq: global objective} over unreliable communication links between the PS and the clients. 
In each round $t$, the communication link between the PS and client $i$ is active with probability $p_i^t$, which could be {\bf time-varying} and is {\bf unknown} to both the PS and the clients.  
We assume that $p_i(t)\ge c$ for all $t$ and all $i$, where $c\in (0,1)$.

\section{A Case Study on the Objective Inconsistency of {FedAvg}}
\label{sec: case study counterexample}
In this section, we use a simple example (a similar setup as in \cite{wang2020tackling}) to illustrate {FedAvg} 
fails to minimize the global objective in 
Eq.\,\eqref{eq: global objective} when $p_i$'s are not uniform.  
For completeness, we formally describe {FedAvg} in Algorithm \ref{alg: fedavg}. 
\begin{algorithm}[h]
\textbf{Input:} $T$, $\x^0$, $s$, $\sth{\eta_t}_{t=0, \cdots, T-1}$ 

The PS and each client initialize parameter $\x^0$; 

\For{$t=0, \cdots, T-1$}
{

\tcc{\color{blue} Let $\calA^t$ denote all the clients with active communication links. }

The PS broadcasts $\x^{t}$ to each client\; %

\For{$i\in [m]$}
{
Draw a fresh sample $\xi_{i}^t$\;

\eIf{$i\in \calA^t$}{ $\x_i^{\pth{t , 0}} \gets \x^{t}$\;}{$\x_i^{\pth{t , 0}} \gets \x_i^{t}$\;}
\For{$k=0, \cdots, s-1$}
{
$\x_i^{(t, k+1)} \gets \x_i^{(t, k)} - \eta_t \nabla \ell_i(\x_i^{(t, k)}; \xi_{i}^t)$\;  
}
$\x_i^{t+1} \gets \x_i^{(t, s)}$\; 
Report $\x_i^{t+1}$ to the PS\; 
}
\tcc{\color{blue} On the PS.}
\eIf{$\calA^t\not=\emptyset$}{$\x^{t+1} \gets \frac{1}{\abth{\calA^t}} \sum_{i\in \calA^t}\x_i^{t+1}$\; }
{
$\x^{t+1} \gets \x^{t}$\; 
}
} 
\caption{Federated Average ({FedAvg}) \cite{mcmahan2017communication}}
\label{alg: fedavg}
\end{algorithm}
Notably, in Algorithm \ref{alg: fedavg}, all the clients (regardless of whether the corresponding links are active or not) compute locally in Algorithm \ref{alg: fedavg} in each round. This is {\it logically equivalent} to the usual setting where only clients in $\calA^t$ %
do the local steps because in line 19 the summation is taken over the clients in $\calA^t$. 
Similar equivalence is observed in \cite{wang2022}.  We present the FedAvg in the form of Algorithm \ref{alg: fedavg} for ease of comparison with our {FedPBC} -- an algorithmic fix to {FedAvg} for bias correction.

Let the local objective $F_i\pth{\x} = \frac{1}{2} \norm{\x - \bu_i}^2,$ where $\bu_i\in \reals^d$ is an arbitrary vector. The corresponding global objective is thus 
\begin{align}
F\pth{\x} = \frac{1}{m} \sum_{i=1}^m F_i\pth{\x} = \frac{1}{2m} \sum_{i=1}^m \norm{\x - \bu_i}^2,\label{eq: counterexample global objective}
\end{align}
with unique 
minimizer %
$\x^\star = \frac{1}{m} \sum_{i=1}^m \bu_i.$ 
\begin{proposition}
\label{proposition: nonuniform}
Choose $\x^0 = \bm{0}$ and $\eta_t = \eta \in (0,1)$ for all $t$. 
For a global objective as per Eq.\,\eqref{eq: counterexample global objective}, if $p_i^t=p_i$ for all $t$, 
under {FedAvg} with exact local gradients %
\begin{small}
\[
\lim_{T\diverge}\x^T = \sum_{i=1}^m\frac{p_i\bu_i \qth{1
+\sum_{j=2}^{m} \pth{-1}^{j+1} \frac{1}{j}\sum_{ 
S \in \calB_j} \prod_{z\in S} p_z}}{1-\Pi_{i=1}^m \pth{1-p_i}},
\]
\end{small}
where $\calB_j \triangleq \sth{S\Big|S\subseteq [m]\setminus\sth{i}, \abth{S} = j-1}.$
\end{proposition}
The proof of Proposition \ref{proposition: nonuniform} can be found in Appendix. 
It can be checked that if there exist $i, i^{\prime}\in [m]$ such that $p_i\not=p_{i^{\prime}}$,  then $\lim_{t\diverge}\x^t\not= \frac{1}{m}\sum_{i=1}^m \bm{u}_i \triangleq \x^*$; when $p_i=p$ for all $i\in [m]$, then $\lim_{t\diverge}\x^t =\x^*$. 
In fact, the output of FedAvg may be arbitrarily away from $\x^{\star}$ depending on $p_i$'s and $\bu_i$'s.

\section{Algorithm: {FedPBC}} 
In this section, we propose {FedPBC} ({\em Federated Postponed Broadcast}, formally described in Algorithm \ref{alg: fedavg variant}) - a simple variant of {FedAvg}. 
\begin{algorithm}[h]
\textbf{Input:} $T$, $\x^0$, $s$, $\sth{\eta_t}_{t=0, \cdots, T-1}$ 

The PS and each client initialize parameter $\x^0$; 

\For{$t=0, \cdots, T-1$}
{
\tcc{\color{blue} Let $\calA^t$ denote all the clients with active communication links\; }

\For{$i\in [m]$}
{
Draw a fresh sample $\xi_{i}^t$\;
$\x_i^{(t,0)} = \x_i^{t}$\; 
\For{$k=0, \cdots, s-1$}
{
$\x_i^{(t, k+1)} = \x_i^{(t, k)} - \eta_t \nabla \ell_i(\x_i^{(t, k)}; \xi_{i}^t)$\;  
}
$\x_i^{t+1} = \x_i^{(t, s)}$\; 
Report $\x_i^{t+1}$ to the PS\; 
} 
\tcc{\color{blue} On the PS.}
\eIf{$\calA^t\not=\emptyset$}{$\x^{t+1} \gets \frac{1}{\abth{\calA^t}} \sum_{i\in \calA^t}\x_i^{t+1}$\; }
{
$\x^{t+1} \gets \x^{t}$\; 
}
Multi-cast $\x^{t+1}$ to each client $i\in \calA^t$\; 
\For{$m\in \calA^t$}
{
$x_i^{t+1} \gets \x^{t+1}$\; 
}
}
\caption{{FedPBC}}
\label{alg: fedavg variant}
\end{algorithm}

The key difference of {FedPBC} from {FedAvg} is that we postpone the global model broadcasts to $\calA^t$ till the end of each round. Postponing the global model broadcast introduces some staleness as the clients might start from different $\x_i^t$ rather than $\x^t$. It turns out that such staleness helps in mitigating the bias caused by non-uniform link activation probabilities.
Moreover, the staleness is mild and there is no significant slowdown. Theoretical analysis and numerical results can be found in Sections \ref{sec: convergence results} and \ref{sec: numerical}, respectively.

\vskip 0.6\baselineskip 
\noindent{\bf Implicit gossiping among clients $\calA^t$. }
From line 14 to line 22 of Algorithm \ref{alg: fedavg variant}, via the coordination of the PS, the clients in $\calA^t$ {\em implicitly} average their local updates with each other, i.e., there is implicit gossiping among the clients in $\calA^t$ at round $t$. 
Formally, we are able to construct a mixing matrix $W^{(t)}$ as
\begin{align*}
W_{ij}^{(t)} =
\begin{cases} 
\frac{1}{\abth{\calA^t}}, &~~~~ \text{if }i, j\in \calA^t; \\    
1, &~~~~ \text{if } i=j\,\text{and} \,\sth{i\notin \calA^t};\\ %
0, &~~~~ \text{otherwise}. 
\end{cases}
\end{align*}
The matrix is by definition {\it doubly-stochastic} and $W^{(t)} = \identity$ when $\calA^t =\emptyset$ or $|\calA^t| =1$.   
We further note that this matrix can be {\it time-varying} even in expectation since the link activation probabilities $p_i^t$'s can be {\it time-varying}. 
As can be seen later, %
this mixing matrix bridges the gap between local and global model heterogeneity and establishes a consensus among different clients.

Let $M^{(t)}: = \expect{\pth{W^{(t)}}^2}$ and $\allones : = \frac{1}{m} \Indc \Indc^\top.$ Define %
as 
 \begin{align}
\label{eq: mixing expected rate}
\rho(t):=\lambda_2\pth{M^{(t)}} ~~~~ \text{and} ~~~ \rho := \max_t \rho(t). 
\end{align}
\begin{lemma}[Ergodicity]
\label{lemma: ergodicity}
Recall that $p_i^t\ge c$ for some constant $c\in (0,1)$. For each $t\ge 1$, it holds that 
$\rho \le 1 - \frac{c^4\qth{1-\pth{1-c}^m}^2}{8}$. 
\end{lemma}
We defer the proof of Lemma \ref{eq: mixing expected rate} to Appendix. 
The following lemma will be used in the convergence analysis. 
\begin{lemma}
\label{lemma: average mixing}
For any matrix $B\in \reals^{d\times m}$, it holds that 
\[
\expect{\fnorm{B\pth{\prod_{r=1}^t W^{(r)} - \allones}}^2} \le \rho^t\|B\|^2_F.  
\] 
\end{lemma}
The proof of Lemma \ref{lemma: average mixing} follows the same outline as that in \cite[Lemma]{wang2022matcha}; it is deferred to Appendix.

\begin{remark}
In Algorithm \ref{alg: fedavg variant},  each client does local computations even if its communication link is not active. 
Continuous local updates appear to be crucial. 
Numerical examples in Section \ref{sec: numerical} show that bias persists when only the active clients do local computations. 
We leave as a future direction on how to remove the bias while maintaining local computation. 
\end{remark}

\section{Convergence Results}
\label{sec: convergence results}
\subsection{Assumptions}
Before diving into our convergence results, we will introduce some assumptions, which are commented towards the end of this subsection. 
\begin{assumption}[Smoothness]
\label{ass: 2 smmothness}
Each local gradient function
$\nabla \ell_{i}(\theta)$ is $L_i$-Lipschitz, \ie,

\[
    \norm{\nabla \ell_{i}(\x_1)-\nabla \ell_{i}(\x_2)}\le L_i \norm{\x_1-\x_2},
\]
for all $\x_1, \x_2,$ and $i\in[m]$. 
Let $L \triangleq \max\limits_{i\in[m]} L_i$. 
\end{assumption}
\begin{assumption}[Bounded Variance]
\label{ass: bounded variance client-wise}
Stochastic gradients at each client node $i\in[m]$ are unbiased estimates of the true gradient of the local objectives, i.e.,  
\[
\expect{\nabla \ell_i(\x_i^t) \mid \calF^{t}}=\nabla F_i(\x_i^t),  
\]
and the variance of stochastic gradients at each client node $i\in[m]$ is uniformly bounded, i.e., 
\[
\expect{\norm{\nabla \ell_i(\x)-\nabla F_i(\x)}^2}\le\sigma^2, 
\]
where $\calF^{t}$ denotes the sigma algebra generated by all the randomness up to iteration $t$. 
\end{assumption}
\begin{assumption}
\label{ass: lower bounds}
There exists $F^*\in \reals$ such that $F(\x)\ge F^*$ for all $\x\in \reals^d$. 
\end{assumption}
\begin{assumption}[Bounded Inter-client Heterogeneity]
\label{ass: bounded similarity}
\begin{align*}
    \frac{1}{m}\sum_{i=1}^m \norm{\nabla F_i(\x)- \nabla F(\x)}^2 \le \beta^2 \norm{\nabla F(\x)}^2+ \zeta^2. 
\end{align*}
\end{assumption}

Assumptions, \ref{ass: 2 smmothness},  \ref{ass: bounded variance client-wise} and \ref{ass: lower bounds} are standard in FL analysis \cite{karimireddy2020scaffold,li2020federated,yuan2022}.
Assumption \ref{ass: bounded similarity} captures the heterogeneity across different users, and it is a more relaxed version (\eg, than \cite{li2020federated,yu2019parallel,wang2021cooperative}.)
Notably, different from \cite{gu2021fast}, we do not assume fresh data per local update, and the unbiasedness in Assumption \ref{ass: bounded variance client-wise} is imposed for global rounds only. 

\subsection{Results}
In this section, we formally state our key lemmas and main theorem. 
All missing proofs can be found in Appendix.
\begin{lemma}[Bounded Local Perturbation]
\label{lemma: local step perturbation}
For $s\ge 1$, we have for all $\x\in\reals^d:$

\begin{small}
\[
\norm{\sum_{k=0}^{s-1}\qth{ \nabla \ell_i (\x^{\pth{t,k}}) - \nabla \ell_i (\x^t)}}
\le \kappa \eta \binom{s}{2} L_i \norm{\nabla \ell_{i}(\x^t)},
\]
\end{small}

where
$
\kappa \triangleq \max_{i}\frac{(1+\eta L_i)^s - 1- s \eta L_i}{\binom{s}{2}\pth{\eta L_i}^2}.
$
\end{lemma}
\begin{claim}
\label{claim: monotonicity}
For any $s\in \naturals$, 
    \begin{align*}
        \kappa &\triangleq \frac{(1+\eta L)^s - 1- s \eta L}{\binom{s}{2}\pth{\eta L}^2}
    \end{align*}
    is monotonic non-decreasing with respect to $\eta>0$. 
\end{claim}
\begin{remark}
Lemma \ref{lemma: local step perturbation} comes from a companion work.
It yields a simple upper bound on the perturbations incurred by multiple local steps.
For the special case when $s=1$, we simply have $\kappa=0$.
For $s\ge 2$, we always have $\kappa \ge 1,$
and furthermore $\kappa \le \frac{e^c-1-c}{c^2/2},$ when $\eta\le \frac{c}{sL},$
which follows from Claim \ref{claim: monotonicity}.
In other words, we can treat $\kappa$ as a constant as long as $\eta$ is sufficiently small. 
\end{remark}
Let 
\begin{align}
\label{eq: average estimate}
\bar{\x}^t \triangleq \frac{1}{m}\sum_{i=1}^m \x_i^t. 
\end{align}
\begin{lemma}[Descent Lemma]
\label{lemma: descent lemma}
Suppose Assumptions \ref{ass: 2 smmothness}, \ref{ass: bounded variance client-wise}, and \ref{ass: bounded similarity} hold, 
under a choice of the learning rate $\eta \le \frac{1}{2s},$ 
the following property holds for $t\ge 0:$
\begin{small}
\begin{align*}
&\expect{F(\bar{\x}^{t+1})  -  F(\bar{\x}^{t}) \mid \calF^{t}}\\ 
& \le - \sth{\frac{s\eta}{4} - 3 \eta^2 s^2 \pth{\beta^2 + 1}\qth{\kappa^2 L^2 + 2L \pth{1 + \frac{\kappa^2 L^2}{4}}}} \norm{\nabla F(\bar{\x}^t)}^2\\
\nonumber
& \qquad  +  3\xi^2 \eta^2 s^2 \qth{\kappa^2 L^2 + 2L \pth{1 + \frac{\kappa^2 L^2}{4}}} \\
& \qquad + \sigma^2 \eta^2 s^2 \qth{\kappa^2 L^2 + 2 L \pth{\frac{1}{m} + \frac{\kappa^2 L^2}{4}}}\\
& \qquad +  \sth{\eta s L^2 + 3\eta^2 s^2 L^2\qth{\kappa^2 L^2 + 2L \pth{1 + \frac{\kappa^2 L^2}{4}}}} \underbrace{\frac{1}{m} \sum_{i=1}^m \norm{\x_i^t - \bar{\x}^t}^2}_{\text{consensus error}}.
\end{align*}
\end{small}
\end{lemma}
\begin{remark}
Lemma \ref{lemma: descent lemma} can be proved via following the standard outline of SGD convergence analysis with non-convex functions and plugging in Lemma \ref{lemma: local step perturbation} to bound the perturbation arises from multiple local updates and non-fresh data per update. 
The consensus error term comes from Assumption \ref{ass: 2 smmothness} and enables us to connect 
our analysis of the aforementioned $W$ matrix, where we borrow the insights from the analysis of gossiping algorithms.
Formally, in matrix form, we use the following notions
\begin{align*}
\bm{X}^{(t)} & = \qth{\x_1^t, \cdots, \x_m^t};   \\
\bm{G}_0^{(t)} &= s\qth{\nabla \ell_1(\x_1^{(t,0)}), \cdots,  s\nabla \ell_m(\x_m^{(t,0)})};\\
\bm{G}^{(t)} & = \qth{\sum_{r=0}^{s-1}\nabla \ell_1(\x_1^{(t,r)}), \cdots,  \sum_{r=0}^{s-1}\nabla \ell_m(\x_m^{(t,r)})} ;\\
\nabla \bm{F}^{(t)} & = \qth{\nabla F_1(\x_1^t), \cdots,  \nabla F_m(\x_m^t)}. 
\end{align*}
Equivalently, we can write down the consensus error in matrix form,
\begin{align*}
\sum_{i=1}^m \norm{\bar{\x}^{t} - \x_i^t}^2 
&= \fnorm{\bm{X}^{(t)} \pth{\identity - \allones}}^2\\
&= \fnorm{\pth{\bm{X}^{(t-1)} - \eta \bm{G}^{(t-1)}} W^{(t-1)} \pth{\identity - \allones}}^2 \\
&= \eta^2 \fnorm{\sum_{q=0}^{t-1} \bm{G}^{(q)} \pth{\Pi_{l=q}^{t-1} W^{(q)} - \allones}}^2,
\end{align*}
where the last follows from the fact that all clients are initiated at the same weights.
\end{remark}

\begin{lemma}[Consensus Error]
\label{lemma: consensus}
Suppose the conditions in Lemma \ref{lemma: descent lemma} are met,
under a choice of the learning rate,
$\eta \le \min\sth{\frac{1}{2s}, \frac{\sqrt{2}}{\kappa sL},\frac{1-\sqrt{\rho}}{6\sqrt{2\rho}Ls^2}}$
the following property holds,
\begin{align*}
&\frac{1}{mT}\sum_{t=0}^{T-1}\expect{\fnorm{\bm{X}^{(t)} \pth{\identity - \allones}^2}}\\
& \le 6 s^2 \eta^2 \sigma^2 \qth{\frac{2\rho }{\pth{1-\sqrt{\rho}}^2}
+ {\frac{\rho }{1-\rho} }}+\frac{72 \xi^2 \eta^2 s^4 \rho}{\pth{1-\sqrt{\rho}}^2}\\
&~ +\frac{72 \pth{\beta^2 + 1} \eta^2 s^4 \rho}{\pth{1-\sqrt{\rho}}^2}\frac{1}{T}\sum_{t=0}^{T-1}\expect{\norm{\nabla F(\bar{\x}^t)}^2}.
\end{align*}
\end{lemma}
Now, we are ready to present our main theorem.
\begin{theorem}
\label{thm: main}
Suppose all the assumptions hold,
and choose a learning rate $\eta = c \sqrt{\frac{m}{sT}}$ for sufficiently large $T$ such that %
\begin{align*}
\eta \le \min \left\{\frac{1}{24\pth{\beta^2 + 1}\mathfrak{C} \qth{1 + \frac{144 s^2 \rho}{\kappa^2 \pth{1-\sqrt{\rho}}^2}}+\frac{1152 \pth{\beta^2 + 1} L s^2 \rho}{\kappa\pth{1-\sqrt{\rho}}^2}},\right.&\\
\left.\frac{1}{2s}, 
\frac{\sqrt{2}}{\kappa sL},\frac{1}{\rho s^3},
\frac{1-\sqrt{\rho}}{6\sqrt{2\rho}Ls^2}\right\}&,
\end{align*}
the following property holds for Algorithm \ref{alg: fedavg variant}

\begin{align*}
&\frac{1}{T}\sum_{k=0}^{T-1}\expect{\norm{\nabla F(\bar{\x}^t)}^2} 
\le O \pth{\frac{8F(\bar{\x}^0) - 8 F^\star }{\sqrt{msT}} }\\
&~+\underbrace{O \pth{  16 L \sqrt{\frac{s}{mT}} \sigma^2
 +
 8 \sqrt{\frac{ms}{T}} \kappa^2 L^2
\pth{1 + \frac{L}{2}} 
 \sigma^2
 }}_{\text{Stochastic gradient noise}}\\ 
&~+\underbrace{O \pth{ 24 \sqrt{\frac{ms}{T}} \qth{
\mathfrak{C}
+ \frac{24 L^2}{\pth{1-\sqrt{\rho}}^2}
} \xi^2 
+\frac{1728 \mathfrak{C} L^2 \xi^2 }{\pth{1-\sqrt{\rho}}^2} \frac{ms}{T}}
}_{\text{Client drift error}}\\
&~+ \underbrace{O \pth{\frac{144  \rho  }{\pth{1-\sqrt{\rho}}^2}   \pth{ L^2 + \frac{3\sqrt{2} L \mathfrak{C}}{\kappa}}\sigma^2 \frac{ms}{T}}}_{\text{Intermittent participation error}},\\
\end{align*}
where
$\mathfrak{C} \triangleq \kappa^2 L^2 + 2 L \pth{1 + \frac{\kappa^2 L^2}{4}}.$
\end{theorem}
\begin{remark}
Here, we remark on Theorem \ref{thm: main}:
\begin{enumerate}
    \item {\bf On the structures.} 
    Except for the first term, the remained terms can be grouped into 
    three parts: the noise introduced by stochastic gradient, and the errors due to client drift (heterogeneity) and intermittent participation, each scaling with a different rate. 
    To control the errors, we need a sufficiently small learning rate $\eta$ that meets all the conditions mentioned above.
    \item {\bf On stationary points of $F$.} 
    Theorem \ref{thm: main} says that $\bar{\x}^t$ in {FedPBC} converges to a stationary point of $F$ asymptotically. 
    In other words, the bias will be corrected towards the end. 
    In contrast, we show in Proposition \ref{proposition: nonuniform} that $\bar{\x}^t$ in {FedAvg} converges to a point that could be arbitrarily far away from the true optimum depending on $p_i^t$ and data heterogeneity.
    \item {\bf On the role of the activation lower bound $c$.} Recall that it has been shown in Lemma \ref{lemma: ergodicity} that $\rho \le 1 - \frac{c^4\qth{1-\pth{1-c}^m}^2}{8}.$ A greater $c$ leads to a smaller $\rho$ and thus a tighter bound on $\frac{1}{T}\sum_{t=0}^{T-1} \expect{\norm{\nabla F \pth{\bar{\x}^t}}}.$ 
    Note that {FedPBC} reduces to {FedAvg} with full-client participation when $c=1$. In that case, our convergence rate $O\pth{\frac{1}{\sqrt{msT}}} + O\pth{\sqrt{\frac{ms}{T}}} + O\pth{\frac{ms}{T}}$ matches the {FedAvg} literature  (\eg, in \cite{wang2020tackling}). 
    We further note that because $\frac{\kappa}{s}$ can be treated as a constant, the convergence rate remains. 
    \item {\bf On linear speedup.} It is trivial to see that the first two terms dominate when $T$ is sufficiently large (\eg, $T\ge c_0 m^3 s^3,$ where $c_0$ is some positive constant.) We shall see linear speedup w.r.t. the first term; however, the second term ultimately dominates all. Thus, it is unlikely that our algorithm achieves linear speedup, which is consistent with {FedAvg} literature, \eg, in \cite{Li2020}.
\end{enumerate}
\end{remark}
\section{Numerical Experiments}
\label{sec: numerical}
In this section, we present the numerical evaluations of the proposed algorithm and {FedAvg}. 
In each round, the PS will send an update request to each client. %
Client $i$ will respond with probability $p_i,$ which is unknown to both the PS and clients. This simulates unstable communications.
\newline
\noindent{\bf Counterexample.} 
Here, we have $m = 100$ clients, each doing $30$-steps local computations, 
communicating for $2000$ rounds,
and holding a local loss function $F_i(\x_i) = \frac{1}{2}\norm{\x_i - \bu_i}^2,$
where $\x_i, \bu_i \in \reals^{100},$ 
$\bu_i\sim \calN \pth{ i \Indc, 0.01 \identity },$
and $\x_i^0=\bm{0}$
for all $i\in[m].$
The learning rate $\eta = 0.0003.$
In addition, we let the first $50$ clients respond with probability $p_0,$ whereas the second half with $p_1$ (to be specified later.)
\setlength{\parskip}{0mm}
\begin{figure}[!htb]
\begin{subfigure}[b]{\columnwidth}
\centering
\includegraphics[width=\columnwidth,trim={0 0 0 0.1cm},clip]{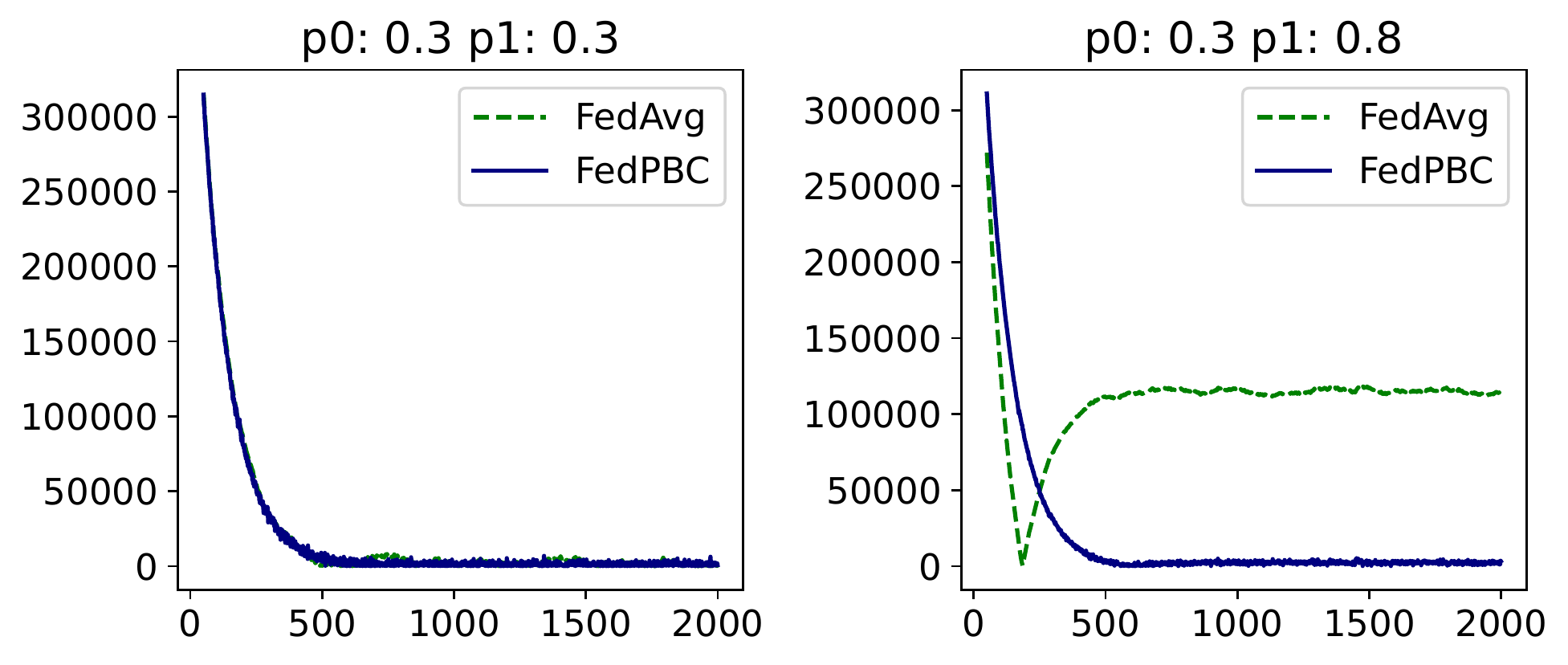}
\vskip -1pt
\caption{Always local computations}
\label{fig: always counter}
\end{subfigure}
\begin{subfigure}[b]{\columnwidth}
\centering
\includegraphics[width=\columnwidth]{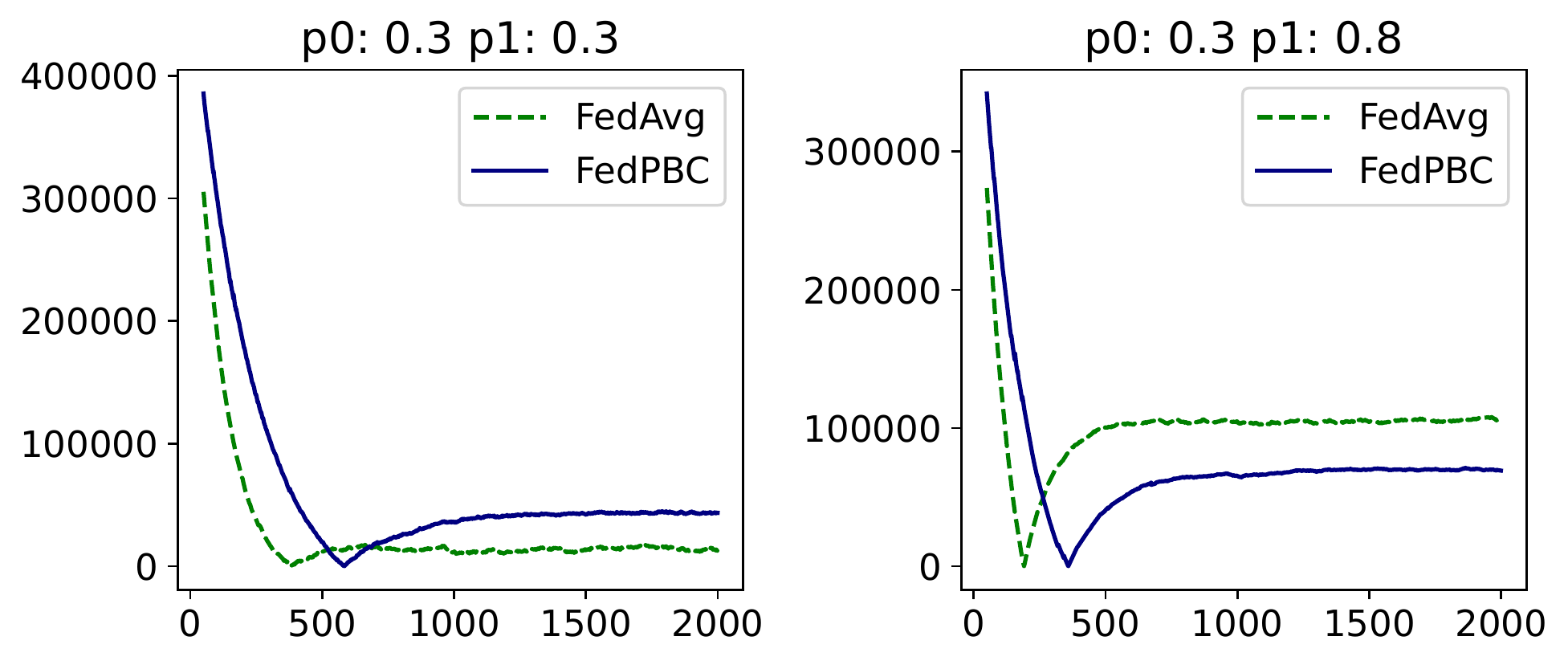}
\vskip -1pt
\caption{Sampled local computations}
\label{fig: sampled counter}
\end{subfigure}
\vskip -1pt
\caption{$\norm{\nabla F \pth{\bar{\x}}}$ evaluation of the counterexample.}
\label{fig: counterexample}
\end{figure}
\setlength{\parskip}{0mm}
\newline
For ease of presentation, we plot the the magnitude of $\norm{\nabla F \pth{\bar{\x}}}$ after the first $50$ communication rounds in Fig.\,\ref{fig: counterexample}. 
Clearly, {FedPBC} is unbiased and converges to the global optimum $\frac{1}{m}\sum_{i=1}^m \bu_i$ in all the combinations of $p_0$ and $p_1$ as suggested by Fig.\,\ref{fig: always counter}, while {FedAvg} will instead converge to a different point seen from non-zero $\nabla F\pth{\bar{\x}}$ when $p_0 \neq p_1$. When $p_0=p_1,$ the two algorithms will converge to the same point, which is the global minimizer and matches our analysis.
In a sharp contrast, if we let only the sampled clients do local computations, the bias persists, which we leave as a future direction.
\begin{figure}[!htb]
\begin{subfigure}[b]{0.49\columnwidth}
\centering
\includegraphics[width=\columnwidth,trim={0 0 0 0.1cm},clip]{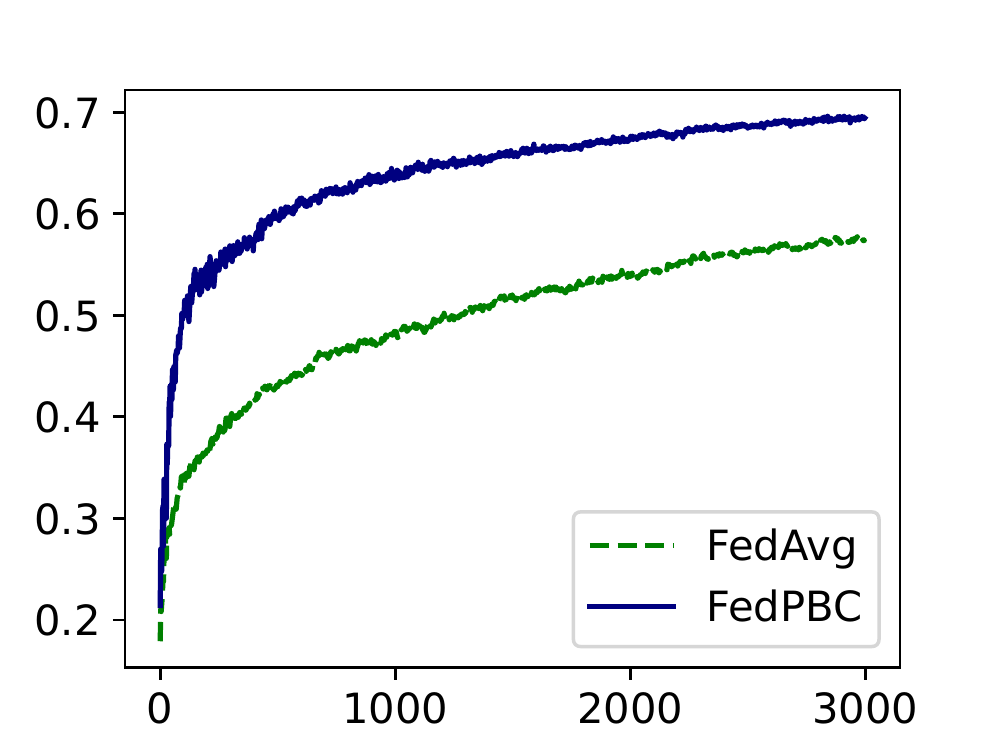}
\vskip -1pt
\caption{Test accuracy}
\label{fig: test acc}
\end{subfigure}
\begin{subfigure}[b]{0.49\columnwidth}
\centering
\includegraphics[width=\columnwidth]{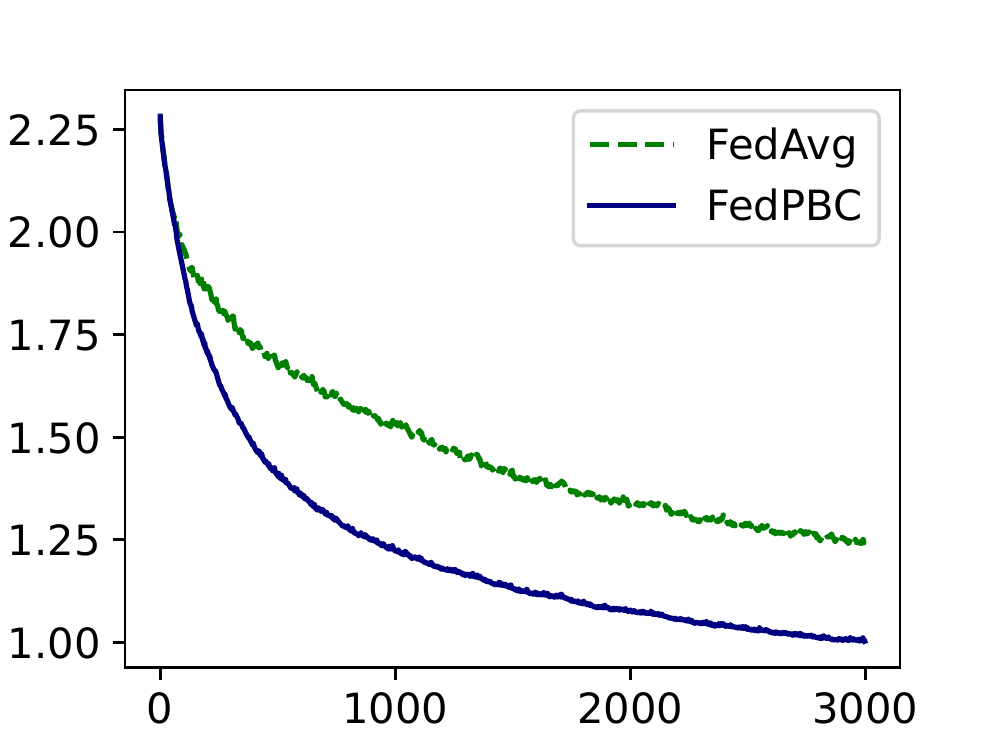}
\vskip -1pt
\caption{Train loss}
\label{fig: train loss}
\end{subfigure}
\vskip -1pt
\caption{Synthetic $(1,1)$ evaluations.}
\label{fig: synthetic}
\end{figure}
\setlength{\parskip}{0mm}
\newline
{\bf Synthetic $(1,1)$ data .}
In this simulation, we first follow \cite{li2020federated} and construct Synthetic $(1,1)$ dataset as follows: we generate samples $(X_i,Y_i)$ for each client $i$ according to the model $y=\argmax \pth{\text{softmax}\pth{W x + b}},$ where $x\in\reals^{60},$ $W\in\reals^{10\times 60},$ $b\in\reals^{10}.$ 
To characterize the non-\iid data, we let $W_i \sim N\pth{ u_i,1},~b_i\sim N\pth{u_i,1},~u_i\sim N\pth{0,\alpha=1},$ and $x_i\sim \calN\pth{v_i,\Sigma},$ where the covariance matrix is diagonal with $\sum_{j,j}=j^{-1.2}.$
Each element in the mean vector $v_i$ is drawn from $N\pth{B_i,1},$ where $B_i\sim N\pth{0,\beta=1}.$

For the non-uniform link activation probabilities $p_i$s, let
$
\prob{Z^t=k} \triangleq  \frac{k^{-a}}{\zeta\pth{a}},
$
where $k\ge 1,~a = 3,$ and $\zeta$ is the Riemann Zeta function.

Define $P_i^t \triangleq \sum_{i=1}^n \indc{Z^t = i},$
where $n = 20000,$
and let $p_i^t \triangleq \frac{P_i^t}{\sum_{i=1}^m P_i^t}$
in each communication round $t.$ 
This makes $p_i^t$ highly non-uniform among different workers.
We note that $p_i^t$s are generated per communication round $t$ and thus {\it time-varying.}
All $p_i^t$s are further clipped to make sure a lower bound $c=0.1$ is met.
The other auxiliary hyper-parameters are set as:
client size $m=150,~\eta = 0.005,$ batch size: $32,$ local computation rounds: $10$ for each $i\in[m],$ communication rounds: $3000.$
Fig.\,\ref{fig: synthetic} shows that {FedPBC} consistently outperforms {FedAvg}. %

\bibliographystyle{IEEEtran}
\bibliography{biblo}

\onecolumn

\appendix
\
\section{Proofs and auxiliary results}
\label{sec: proofs}

\begin{proof}[\bf Proof of Proposition \ref{proposition: nonuniform}]
At each client $i\in\calA^t,$ we have
\begin{align*}
\x_i^{(t,k+1)} 
& = \pth{1-\eta}^{k+1} \x^t + \eta \bu_i \qth{\sum_{r=0}^k (1-\eta)^r}. 
\end{align*}
Using the convention that $\frac{0}{0}=0$, we get 
\begin{align*}
\x^{t+1}
& = \x^t\indc{\calA_t=\emptyset} 
+ \pth{1 - \eta}^{s}\x^t\indc{\calA^t\neq \emptyset} %
+ \frac{\eta \sum_{i\in\calA^t}\bu_i \qth{\sum_{r=0}^{s-1} \pth{1-\eta}^r}\indc{\calA^t\neq\emptyset}}{\abth{\calA^t}}\\
& = \qth{\indc{\calA^t = \emptyset}+\pth{1-\eta}^{s} \indc{\calA^t \neq \emptyset}}\x^t %
+ \eta \qth{\sum_{k=0}^{s-1} (1-\eta)^k} \frac{\indc{\calA^t \neq \emptyset}}{\abth{\calA^t}}\sum_{i\in \calA^t}\bu_i. 
\end{align*}
Let $\eta_t = \eta$ for all $t$. 
Since $p_i^t=p_i$ for all $i\in [m]$, 
\[
\expect{\frac{1}{\abth{\calA^t}}\sum_{i\in \calA^t}\bu_i \Big| \calA^t \neq \emptyset} = \expect{\frac{1}{\abth{\calA^1}}\sum_{i\in \calA^1}\bu_i \Big| \calA^1 \neq \emptyset} 
\]
holds for all $t$. 
Taking expectation w.r.t. $\calA^t$, we get 
\begin{align*}
\x^{t+1}  &= \qth{\prob{\calA^t = \emptyset}+\pth{1-\eta}^{s} \prob{\calA^t \neq \emptyset}} \x^t  %
+ \eta \qth{\sum_{k=0}^{s-1} (1-\eta)^k} \expect{\frac{\sum_{i\in \calA^t}\bu_i}{\abth{\calA^t}}\Big| \calA^t \neq \emptyset} \prob{\calA^t \neq \emptyset}\\
& = \pth{1 - {\rma}^{t+1}} \expect{\frac{1}{\abth{\calA^1}}\sum_{i\in \calA^1}\bu_i \Big| \calA^1 \neq \emptyset}, 
\end{align*}
where we use the fact that $\x^0=\bm{0}$, and 
\begin{align*}
\rma& = \Pi_{i=1}^m\pth{1-p_i} + \qth{1 -\Pi_{i=1}^m\pth{1-p_i}}\pth{1-\eta}^s. 
\end{align*}
Since $\rma<1$, we get 
$
\lim_{t\rightarrow\infty} 1 - \rma^{t+1} = 1. 
$
Let $X_i = \indc{i\in \calA^1}$ for each $i\in [m]$. 
In sequel, we alternatively state the event $\sum_{i=1}^m X_i \neq 0$ as $\calA^1 \neq \emptyset$ since they are equivalent.
\begin{small}
\begin{align*}
&\expect{\frac{\sum_{i\in \calA^1}\bu_i}{\abth{\calA^1}} \Big| \calA^1 \neq \emptyset} 
 =\expect{\frac{\sum_{i=1}^m X_i \bu_i}{\sum_{i=1}^m X_i} \Big| \calA^1 \neq \emptyset} \\
&= \expect{\sum_{i=1}^m \frac{X_i}{\sum_{i=1}^m X_i} \bu_i \Big| \calA^1 \neq \emptyset} = \sum_{i=1}^m \bu_i\expect{ \frac{X_i}{\sum_{j=1}^m X_j} \Big| \calA^1 \neq \emptyset}.
\end{align*}
\end{small}
Using the convention that $\frac{0}{0}=0$, we know that
\begin{align*}
& \expect{\frac{X_i}{\sum_{j=1}^M X_j}  \Big| \sum_{j=1}^M X_j \neq 0} \\
& = \frac{\expect{\frac{X_i}{\sum_{i=1}^M X_i}  \Big| \calA^1 \neq \emptyset}\prob{\calA^1 \neq \emptyset}  + 0 \times\prob{\calA^1 = \emptyset}}{\prob{\calA^1 \neq \emptyset}}\\
& = \frac{1}{1-\Pi_{i=1}^m\pth{1-p_i}}\expect{ \frac{X_i}{\sum_{j=1}^m X_j}}.
\end{align*}
Additionally, 
\begin{align*}
\expect{ \frac{X_i}{\sum_{i=1}^m X_i}} 
&= \prob{X_i = 1}\expect{ \frac{X_i}{\sum_{j=1}^m X_j}\Big| X_i = 1} \\
&+ \prob{X_i = 0}\expect{ \frac{X_i}{\sum_{j=1}^m X_j} \Big| X_i = 0}\\
&= p_i\expect{ \frac{1}{1 + \sum_{j\in [m]\setminus \sth{i}} X_j}\Big| X_i = 1}\\
&=  p_i 
+\sum_{j=2}^{m} \pth{-1}^{j+1} \frac{p_i}{j}\sum_{ 
S\in \calB_j} \prod_{z\in S} p_z,
\end{align*}
where $\calB_j \triangleq \sth{S \Big|S\subseteq [m]\setminus\sth{i}, \abth{S} = j-1
},$ and the last follows from the definition of a binomial distribution and can be seen through inspection of the terms.
\end{proof}

\begin{proof}[\bf Proof of Lemma \ref{lemma: ergodicity}]
For ease of exposition, in this proof we drop the time index.

We first get the explicit expression for $\expect{W^2_{jj^{\prime}}\mid \calA \neq \emptyset}$. 
For $j^{\prime}\not=j$, we have 
\begin{align*}
W^2_{jj^{\prime}} & = \sum_{k=1}^m W_{jk}W_{j^{\prime}k} \\
& = W_{jj}W_{j^{\prime}j} + W_{jj^{\prime}}W_{j^{\prime}j^{\prime}} + \sum_{k\in [m]\setminus \{j, j^{\prime}\}} W_{jk}W_{j^{\prime}k}.   
\end{align*}
When $k\not=j$ and $k\not=j^{\prime}$, we have 
\begin{align*}
W_{jk}W_{j^{\prime}k} & = \frac{1}{|\calA|^2} \indc{j\in \calA} \indc{j^{\prime}\in \calA}\indc{k\in \calA}.     
\end{align*}
In addition, we have 
\begin{align*}
W_{jj}W_{j^{\prime}j} & = \frac{1}{|\calA|} \pth{1-\indc{j\in \calA}} \indc{j\in \calA}\indc{j^{\prime}\in \calA} \\
& + \frac{1}{|\calA|^2} \indc{j\in \calA}\indc{j^{\prime}\in \calA}, 
\end{align*}
and 
\begin{align*}
W_{j^{\prime}j^{\prime}}W_{jj^{\prime}} 
& = \frac{1}{|\calA|} \pth{1-\indc{j^{\prime}\in \calA}} \indc{j\in \calA}\indc{j^{\prime}\in \calA} \\
& + \frac{1}{|\calA|^2} \indc{j\in \calA}\indc{j^{\prime}\in \calA}. 
\end{align*} 
Thus, 
\begin{align*}
\label{eq: doubly matrix: off-diagnal}
& W^2_{jj^{\prime}} = \sum_{k=1}^m W_{jk}W_{j^{\prime}k} \\ 
& = \frac{1}{|\calA|}\indc{j\in \calA} \indc{j^{\prime}\in \calA} + \frac{1}{|\calA|} \pth{1-\indc{j\in \calA}} \indc{j\in \calA}\indc{j^{\prime}\in \calA}\\
&+ \frac{1}{|\calA|} \pth{1-\indc{j\in \calA}} \indc{j\in \calA}\indc{j^{\prime}\in \calA}. 
\end{align*}

For $j=j^{\prime}$, we have 
$W^2_{jj} = \frac{1}{|\calA|}\indc{j\in \calA} + \pth{1-\indc{j\in \calA}}. $
Taking expectation, we get 
\begin{align*}
\expect{W^2_{jj}\mid \calA^t\neq\emptyset} 
& = \expect{\frac{1}{|\calA|}\indc{j\in \calA} + \pth{1-\indc{j\in \calA}}}\\
&= \expect{\frac{1}{1+|\calA \setminus \{j\}|}}p_j + 1\cdot (1-p_j).  
\end{align*}
Note that $\calA \setminus \{j\}$ is random and could be empty. 

Let $X_i = \indc{i\in \calA}$. 
We have 
\begin{align*}
\expect{\frac{1}{1+|\calA \setminus \{j\}|}} & = \expect{\frac{1}{1+\sum_{i\in [m]\setminus \{j\}} X_i}} = \int_{0}^1 \prod_{k\not=j}\qth{(1-p_k) + p_k s}\diff s \\ 
&\ge \int_{0}^1 \prod_{k\not=j}\qth{(1-p_k)s + p_k s}\diff s =\frac{1}{m}. 
\end{align*}
Thus, 
$
\expect{W^2_{jj}\mid\calA\neq \emptyset} \ge \frac{1}{m} p_j + (1-p_j) \ge \frac{1}{m}{ \ge \frac{c^2}{m}}.      
$
Similarly, 
\begin{small}
\begin{align*}
\expect{W^2_{jj^{\prime}}\mid \calA \neq \emptyset} & = p_j p_{j^{\prime}} \expect{\frac{1}{2+ \sum_{k\in [m]\setminus \{j, j^{\prime}\}}X_k }} \ge \frac{ p_j p_{j^{\prime}}}{m} \ge \frac{c^2}{m}. 
\end{align*}
\end{small}
Then, 
\begin{align*}
M_{jj^{\prime}}=\expect{W^2_{jj^\prime}} & = \expect{W^2_{jj^\prime}\mid \calA \neq\emptyset}\prob{\calA \neq\emptyset}\\
&+\expect{W^2_{jj^\prime}\mid \calA =\emptyset}\prob{\calA =\emptyset}\\
&\ge \frac{c^2}{m}\qth{1-\pth{1-c}^m}.
\end{align*}

We first show that $\rho (t) = \lambda_2 (M).$ 
We denote by $\lambda_i$ and $v_i$ the non-increasing eigenvalues and the associated eigenvectors of matrix $M$ for $i\in [m]$ with $\lambda_1 = 1$ and $v_1 = \frac{1}{\sqrt{m}}\Indc.$
By spectral decomposition
\begin{align*}
    M - \frac{1}{m}\bm{1}\bm{1}^{\top} = \sum_{i=1}^{m} \lambda_i v_i v_i^\top - \frac{1}{m}\bm{1}\bm{1}^{\top} = \sum_{i=2}^{m} \lambda_i v_i v_i^\top, 
\end{align*}
showing $\rho(t) = \lambda_2.$

Next, we show that a Markov chain with $M$ as the transition matrix is ergodic. 
This is indeed true as the chain is (1) {\it irreducible}: $M_{j j^\prime}\ge  \frac{c^2}{m}\qth{1-\pth{1-c}^m}>0$ for $j,j^\prime \in[m]$ and (2) {\it aperiodic} (it has self-loops.) Moreover, it has a stationary distribution
$
\pi = \frac{1}{m}\Indc^\top.
$
Furthermore, this irreducible Markov chain is reversible since the following property is satisfied for all the states
$
\pi_i M_{ij} = \pi_j M_{ji}.
$

Following \cite{jerrum1988conductance}, the conductance of reversible Markov chain with underlying graph $\calG$ is defined by
$
\Phi(\calG) = \min_{\sum_{i\in\calS} \pi_i \le \frac{1}{2}} \frac{\sum_{i\in\calS, j\notin \calS} w_{ij}}{\sum_{i\in \calS} \pi_i},
$
where the vertices of the graph are the states of the $M$ Markov chain, and for each pair $i,j\in \calV,$ the edge weight $w_{ij} = M_{ij} \pi_i = M_{ji} \pi_j.$ 
From Cheeger's inequality, we know that
$
\frac{1 - \lambda_2}{2} \le \Phi(\calG) \le \sqrt{2 \pth{1-\lambda_2}},
$
where $\lambda_2$ is the second largest eigenvalue of %
$M.$
It remains to bound $\Phi(\calG),$
\begin{align*}
    \Phi(\calG) 
    &= \min_{\sum_{i\in\calS} \pi_i \le \frac{1}{2}} \frac{\pi_i \sum_{i\in\calS, j\notin \calS} M_{ij}}{\sum_{i\in \calS} \pi_i}\\
    &\ge \frac{\pth{\frac{c}{m}}^2\qth{1-\pth{1-c}^m}\abth{\calS}\abth{\bar{\calS}}}{\frac{\abth{\calS}}{m}} = \frac{c^2\qth{1-\pth{1-c}^m}}{m} \abth{\bar{\calS}},
\end{align*}
where the inequality follows from (1) $\calG$ is fully-connected (2) $M_{j j^\prime}\ge \frac{c^2}{m} \qth{1-\pth{1-c}^m}$ for $j,j^\prime \in[m].$
Meanwhile,
$
\abth{\bar{\calS}} = m - \abth{\calS} \ge \frac{m}{2}.
$
Plug it back in, we get
\begin{align*}
    \Phi(\calG) 
    \ge \frac{c^2\qth{1-\pth{1-c}^m}}{m} \abth{\bar{\calS}} \ge \frac{c^2\qth{1-\pth{1-c}^m}}{2}.
\end{align*}
Thus,
$
    \rho(t) = \lambda_2 \le 1 - \frac{\Phi^2\pth{\calG}}{2} \le 1 - \frac{c^4\qth{1-\pth{1-c}^m}^2}{8}.
$
\end{proof}
\begin{proof}[\bf Proof of Lemma \ref{lemma: average mixing}]
Similar to the proof in \cite{wang2022matcha}, let us define $A_{r,t}\triangleq \prod_{l=r}^t W^{(\ell)} - \allones$ and use $\bm{b}_i^\top$ to denote the $i$-th row vector of $B.$ 
Since for $\ell\in\naturals,$ we have $(W^{(\ell)})^\top = W^{(\ell)}$ and $W^{(\ell)}\allones = \allones W^{(\ell)} = \allones.$
Thus, one can obtain
\[
A_{1,t} = \prod_{\ell=1}^t \pth{W^{(\ell)} - \allones} = A_{1,t-1}\pth{W^{(t)}-\allones}.
\]
Then, by taking expectation w.r.t. $W^{(t)},$ we have 
\begin{align*}
    & \mathbb{E}_{W^{(t)}}\qth{\fnorm{B A_{1,t}}^2}\\
    & = \sum_{i=1}^d \mathbb{E}_{W^{(t)}}\qth{\norm{b_i^\top A_{1,t}}^2}\\
    & = \sum_{i=1}^d \mathbb{E}_{W^{(t)}}\qth{b_i^\top A_{1,t-1} \pth{\pth{W^{(t)}}^\top W^{(t)} -\allones}A_{1,t-1}^\top b_i}\\
    & = b_i^\top A_{1,t-1}\mathbb{E}_{W^{(t)}}\qth{\pth{(W^{(t)})^\top W^{(t)} -\allones}} A_{1,t-1}^\top b_i.
\end{align*}
Let $C_t = \mathbb{E}_{W^{(t)}}\qth{(W^{(t)})^\top W^{(t)} - \allones}$ and $v_i = A_{1,t-1}^\top b_i,$ then
\begin{align*}
\mathbb{E}_{W(t)}\qth{\fnorm{BA_{1,t}}^2} & 
= \sum_{i=1}^d v_i^\top C_t v_i \le \sigma_{\max}\pth{C_t} \sum_{i=1}^d v_i^\top v_i \\
& \le \rho \fnorm{BA_{1,t-1}}^2.
\end{align*}
Repeat the above procedures, since $W^{(\ell)}$'s are independent matrices, we have
\[
\expect{\fnorm{BA_{1,t}}^2} = 
\mathbb{E}_{W^{(1)}}\qth{\mathbb{E}_{W^{(2)}}\qth{\cdots  \mathbb{E}_{W^{(t-1)}}\qth{ \mathbb{E}_{W^{(t)}}\qth{\fnorm{BA_{1,t}}^2}} } } \le \rho^{t}\fnorm{B}^2.
\]
\end{proof}
\begin{proof}[\bf Proof of Lemma \ref{lemma: local step perturbation}]
By the definition of $\kappa$,
$$
\kappa \eta\binom{s}{2} L_i
\ge \frac{(1+\eta L_i)^s - 1- s \eta_t L_i}{\eta L_i}.
$$
Hence it suffices to show 
\begin{equation}
\label{eq:s-step-diff-proof}
\norm{\sum_{k=0}^{s-1}\qth{ \nabla \ell_i (\x^{\pth{t,k}}) - \nabla \ell_i (\x^t)}}
\le \frac{(1+\eta L_i)^s - 1- s \eta L_i}{\eta L_i} \norm{\nabla \ell_{i}(\x_t)}.
\end{equation}
We prove~\eqref{eq:s-step-diff-proof} holds for all $s\ge 1$ by induction.
The base case $s=1$ follows from the definition.
Suppose~\eqref{eq:s-step-diff-proof} holds true for $s=1,\dots,n-1$, where $n\ge 2$.
Next we prove~\eqref{eq:s-step-diff-proof} for $s = n$.
We have
\begin{align}
 \norm{ \nabla \ell_i (\x^{\pth{t,n-1}}) - \nabla \ell_i (\x^t)}
 &\le L_i \norm{\x^{\pth{t,n-1}} - \x^t}\nonumber\\
 &\le L_i \eta \norm{\sum_{k=0}^{n-2}\qth{\nabla \ell_i \pth{\x^{\pth{t,k}}}-\nabla \ell_i \pth{\x^{t}}}} + L_i \eta \pth{n-1} \norm{\nabla \ell_i\pth{\x^t}}\nonumber\\
 &\overset{\pth{\rma}}{\le} \qth{\pth{1+\eta L_i}^{n-1} - 1}\norm{\nabla \ell_i\pth{\x^t}}, \label{eq: induction}
\end{align}
where $\pth{\rma}$ follows from the induction hypothesis.

Plug Eq.\,\eqref{eq: induction} back in, use the induction hypothesis and triangle inequality, we get
\begin{align*}
\norm{\sum_{k=0}^{n-1}\qth{ \nabla \ell_i (\x^{\pth{t,k}}) - \nabla \ell_i (\x^t)}}
&\le  \norm{\sum_{k=0}^{n-2}\qth{ \nabla \ell_i (\x^{\pth{t,k}}) - \nabla \ell_i (\x^t)}} + \norm{ \nabla \ell_i (\x^{\pth{t,n-1}}) - \nabla \ell_i (\x^t)}\\
& \le \qth{\frac{\pth{1+\eta L_i}^{n-1} - 1 - \pth{n-1}\eta L_i}{\eta L_i} + \pth{1+\eta L_i}^{n-1} - 1}\norm{\nabla \ell_i\pth{\x^t}}\\
& = \frac{\pth{1+\eta L_i}^{n} - 1 - n \eta L_i}{\eta L_i}\norm{\nabla \ell_i\pth{\x^t}}.
\end{align*}
The proof is completed.
\end{proof}
\begin{proof}[\bf Proof of Claim \ref{claim: monotonicity}]
Recall that 
\begin{align*}
\kappa &= \frac{(1+\eta L)^s - 1- s \eta L}{\binom{s}{2}\pth{\eta L}^2}.
\end{align*}
From binomial theorem, we know that
\[
(1+\eta L)^s = \sum_{i=0}^s \binom{s}{i} \pth{\eta L}^i,
\]
it follows that
\begin{align*}
\frac{(1+\eta L)^s - 1- s \eta L }{\binom{s}{2}\pth{\eta L}^2}&= \frac{\sum_{i=2}^s \binom{s}{i} \pth{\eta L}^i}{\binom{s}{2}\pth{\eta L}^2}
= \sum_{i=2}^s \frac{\binom{s}{i}}{\binom{s}{2}} \pth{\eta L}^{i-2}.
\end{align*}
Since $i-2\ge 0$ for $i\ge 2,$ we can see that $\kappa$ is a polynomial of $\pth{\eta L}.$ 
Thus, it is monotonic non-decreasing w.r.t. $\eta>0.$

The proof is completed.
\end{proof}
\begin{proposition}
\label{proposition: average gradient to global gradient}
For any $t\in[T-1]$, it holds that 
\begin{align*}
&\frac{1}{m}\sum_{i=1}^m\norm{\nabla F_i(\x_i^t)}^2 \le \frac{3L^2}{m} \sum_{i=1}^m \norm{\x_i^t - \bar{\x}^t}^2 + 3\pth{\beta^2 + 1}\norm{\nabla F(\bar{\x}^t)}^2 + 3\xi^2. 
\end{align*}
\end{proposition}
\begin{proof}[\bf Proof of Proposition \ref{proposition: average gradient to global gradient}]

\begin{align*}
\frac{1}{m}\sum_{i=1}^m\norm{\nabla F_i(\x_i^t)}^2  &=  \frac{1}{m}\sum_{i=1}^m\norm{\nabla F_i(\x_i^t) - \nabla F_i(\bar{\x}^t) + \nabla F_i(\bar{\x}^t) - \nabla F(\bar{\x}^t) + \nabla F(\bar{\x}^t)}^2\\
& \le  \frac{3}{m}\sum_{i=1}^m \norm{\nabla F_i(\x_i^t) - \nabla F_i(\bar{\x}^t)}^2 +  \frac{3}{m}\sum_{i=1}^m \norm{\nabla F_i(\bar{\x}^t) - \nabla F(\bar{\x}^t)}^2 + 3 \norm{\nabla F(\bar{\x}^t)}^2\\
& \overset{(a)}{\le} \frac{3L^2}{m} \sum_{i=1}^m \norm{\x_i^t - \bar{\x}^t}^2  + 3 \beta^2 \norm{\nabla F(\bar{\x}^t)}^2 + 3\xi^2 + 3 \norm{\nabla F(\bar{\x}^t)}^2\\
& = \frac{3L^2}{m} \sum_{i=1}^m \norm{\x_i^t - \bar{\x}^t}^2 + 3\pth{\beta^2 + 1}\norm{\nabla F(\bar{\x}^t)}^2 + 3\xi^2, 
\end{align*} 
where inequality (a) follows from Assumptions \ref{ass: 2 smmothness} and \ref{ass: bounded similarity}.

\end{proof}
\begin{proof}[\bf Proof of Lemma \ref{lemma: descent lemma}]
By $L$-smoothness, we have 
\begin{align*}
F(\bar{\x}^{t+1})  -  F(\bar{\x}^{t}) & \le \iprod{\nabla F(\bar{\x}^{t})}{\bar{\x}^{t+1} - \bar{\x}^{t}} + \frac{L}{2}\norm{\bar{\x}^{t+1}- \bar{\x}^{t}}^2 \\
& = \iprod{\nabla F(\bar{\x}^{t})}{- \frac{\eta}{m} \bm{G}^{(t)} \bm{1}} +  \frac{L\eta^2}{2}\norm{\frac{1}{m} \bm{G}^{(t)} \bm{1}}^2.  
\end{align*}
Taking expectations with respect to the randomness in the mini-batches at $k$-th rounds, we have 
\begin{align*}
\expect{F(\bar{\x}^{t+1})  -  F(\bar{\x}^{t}) \mid \calF^{t}} & \le \expect{\iprod{\nabla F(\bar{\x}^{t})}{- \frac{\eta}{m} \bm{G}^{(t)} \bm{1}} +  \frac{L}{2}\norm{ - \frac{\eta}{m} \bm{G}^{(t)} \bm{1}}^2}. 
\end{align*}
For ease of notations, we abbreviate $\nabla \ell_i\pth{\x_i^{\pth{t,k}}}$ as $\nabla \ell_i^{\pth{t,k}}.$
\paragraph{Bounding $\iprod{\nabla f(\bar{\x}^{t})}{- \frac{\eta}{m} \nabla \bm{F}^{(t)} \bm{1}}$.}
\begin{align*}
&\expect{\iprod{\nabla F(\bar{\x}^{t})}{- \frac{\eta}{m} \bm{G}^{(t)} \bm{1}} \mid \calF^{t}} 
 = - \frac{\eta}{m} \expect{\iprod{\nabla F(\bar{\x}^{t})}{\sum_{i=1}^m \sum_{k=0}^{s-1} \nabla \ell_i^{(t,k)}} \mid \calF^{t}}    \\
& = - \frac{\eta}{m} \expect{\iprod{\nabla F(\bar{\x}^{t})}{\sum_{i=1}^m s \nabla \ell_i^{(t,0)} - s \nabla \ell_i^{(t,0)} +\sum_{k=0}^{s-1} \nabla \ell_i^{(t,k)}} \mid \calF^{t}} \\
& =  - \frac{s\eta}{m} \iprod{\nabla F(\bar{\x}^{t})}{\sum_{i=1}^m  \nabla F_i(\x_i^t)} + \expect{\sum_{i=1}^m \frac{\eta}{m} \iprod{\nabla F(\bar{\x}^{t})}{s \nabla \ell_i^{(t,0)} - \sum_{k=0}^{s-1} \nabla \ell_i^{(t,k)}} \mid \calF^{t}} \\
& = \underbrace{- \frac{s\eta}{m} \iprod{\nabla F(\bar{\x}^{t})}{\nabla \bm{F}^{(t)}\bm{1}}}_{(\rmA)} + \underbrace{\expect{\frac{\eta}{m} \iprod{\nabla F(\bar{\x}^{t})}{\sum_{i=1}^m s \nabla \ell_i^{(t,0)} - \sum_{k=0}^{s-1} \nabla \ell_i^{(t,k)} }\mid \calF^{t}}}_{(\rmB)}.  
\end{align*}
Term $(\rmA)$ can be bounded as 
\begin{align*}
\iprod{\nabla F(\bar{\x}^{t})}{- \frac{s\eta}{m} \nabla \bm{F}^{(t)} \bm{1}} 
& = - s\eta\iprod{\nabla F(\bar{\x}^{t})}{ \frac{1}{m} \nabla \bm{F}^{(t)} \bm{1}} \\
& = - \frac{s \eta}{2}\pth{\norm{\nabla F(\bar{\x}^{t})}^2 + \norm{\frac{1}{m}\nabla \bm{F}^{(t)}\bm{1}}^2 - \norm{\nabla F(\bar{\x}^{t}) - \frac{1}{m}\nabla \bm{F}^{(t)}\bm{1}}^2} \\
& =  -\frac{s\eta}{2} \norm{\nabla F(\bar{\x}^{t})}^2 - \frac{s\eta}{2} \norm{\frac{1}{m}\nabla \bm{F}^{(t)}\bm{1}}^2 
+ \frac{s\eta}{2} \norm{\nabla F(\bar{\x}^{t}) - \frac{1}{m}\nabla \bm{F}^{(t)}\bm{1}}^2 \\
& \le  -\frac{s\eta}{2} \norm{\nabla F(\bar{\x}^{t})}^2 - \frac{s\eta}{2} \norm{\frac{1}{m}\nabla \bm{F}^{(t)}\bm{1}}^2 
+ \frac{s\eta L^2 }{2m}\sum_{i=1}^m \norm{\bar{\x}^{t} - \x_i^t}^2.   
\end{align*}

For term (B), we have
\begin{align*}
&\expect{\frac{\eta}{m} \iprod{\nabla F(\bar{\x}^{t})}{\sum_{i=1}^m s \nabla \ell_i^{(t,0)} - \sum_{k=0}^{s-1} \nabla \ell_i^{(t,k)}} \mid \calF^{t}} \\
& = \frac{\eta}{m} \sum_{i=1}^m\iprod{\nabla F(\bar{\x}^{t})}{ \expect{ s \nabla \ell_i^{(t,0)} - \sum_{k=0}^{s-1} \nabla \ell_i^{(t,k)}\mid \calF^{t}}} \\
& \le \frac{\eta}{2m} \sum_{i=1}^m \pth{\eta s^2 \norm{\nabla F(\bar{\x}^{t})}^2 + \frac{1}{\eta s^2} \norm{\expect{ s \nabla \ell_i^{(t,0)} - \sum_{k=0}^{s-1} \nabla \ell_i^{(t,k)}\mid \calF^{t}}}^2 }  \\
& \overset{(\rma)}{\le} \frac{\eta^2 s^2}{2} \norm{\nabla F(\bar{x}^{t})}^2 + \underbrace{\frac{1}{2m s^2} \sum_{i=1}^m \expect{\norm{s \nabla \ell_i^{(t,0)} - \sum_{k=0}^{s-1} \nabla \ell_i^{(t,k)}}^2 \mid \calF^{t}}}_{(\rmB.1)}.  
\end{align*}
From Lemma \ref{lemma: local step perturbation}, we bound term $(\rmB.1)$ as follows  
\begin{align*}
&\frac{1}{2m s^2} \sum_{i=1}^m \expect{\norm{s \nabla \ell_i^{(t,0)} - \sum_{k=0}^{s-1} \nabla \ell_i^{(t,k)}}^2 \mid \calF^{t}} \\
& \le \frac{1}{2m s^2} \sum_{i=1}^m \expect{ \kappa^2 \eta^2 \binom{s}{2}^2 L^2\norm{\nabla \ell_{i}^{\pth{t,0}}}^2 \mid \calF^{t}} \\
& = \frac{\kappa^2 \eta^2 \binom{s}{2}^2 L^2}{2m s^2} \sum_{i=1}^m \expect{\norm{\nabla \ell_{i}^{\pth{t,0}}}^2 \mid \calF^{t}}\\
& = \frac{\kappa^2 \eta^2 \binom{s}{2}^2 L^2}{2m s^2} \sum_{i=1}^m \expect{\norm{\nabla \ell_{i}^{\pth{t,0}} - \nabla F_i(\x_i^t) + \nabla F_i(\x_i^t)}^2 \mid \calF^{t}} \\
& \le \frac{\kappa^2 \eta^2 \binom{s}{2}^2 L^2}{m s^2} \sum_{i=1}^m \expect{\norm{\nabla \ell_{i}^{\pth{t,0}} - \nabla F_i(\x_i^t)}^2\mid \calF^{t}} + \norm{\nabla F_i(\x_i^t)}^2\\
& \overset{(\rma)}{\le} \kappa^2 \eta^2 s^2 L^2 \sigma^2 + \frac{\kappa^2 \eta^2 s^2 L^2}{m} \sum_{i=1}^m\norm{\nabla F_i(\x_i^t)}^2\\
& \le \kappa^2 \eta^2 s^2 L^2 \frac{3L^2}{m} \sum_{i=1}^m \norm{\x_i^t - \bar{\x}^t}^2 + 3\kappa^2 \eta^2 s^2 L^2\pth{\beta^2 + 1}\norm{\nabla F(\bar{\x}^t)}^2 + \kappa^2 \eta^2 s^2 L^2 (3\xi^2 +\sigma^2), 
\end{align*}
where inequality $(\rma)$ follows from Assumption \ref{ass: bounded variance client-wise}, and the last inequality follows from Proposition \ref{proposition: average gradient to global gradient}. 
Thus, term $(\rmB)$ can be further bounded as 
\begin{align*}
&\expect{\frac{\eta}{m} \iprod{\nabla F(\bar{\x}^{t})}{\sum_{i=1}^m s \nabla \ell_i^{(t,0)} - \sum_{k=0}^{s-1} \nabla \ell_i^{(t,k)}} \mid \calF^{t}} \\
& \le \frac{\eta^2 s^2}{2} \norm{\nabla F(\bar{\x}^{t})}^2 + \frac{3L^4 \eta^2 \kappa^2 s^2 }{m} \sum_{i=1}^m \norm{\x_i^t - \bar{\x}^t}^2 + 3\kappa^2 \eta^2 s^2 L^2\pth{\beta^2 + 1}\norm{\nabla F(\bar{\x}^t)}^2 + \kappa^2 \eta^2 s^2 L^2 (3\xi^2 +\sigma^2). 
\end{align*}
Combing the bounds of terms $(\rmA)$ and $(\rmB)$, we get 
\begin{align}
\label{eq: bound 1}
\nonumber
\expect{\iprod{\nabla F(\bar{\x}^{t})}{- \frac{\eta}{m} \bm{G}^{(t)} \bm{1}} \mid \calF^{t}} 
&\le  - \qth{\frac{s\eta}{2} - \frac{\eta^2 s^2}{2} - 3\kappa^2 \eta^2 s^2 L^2\pth{\beta^2 + 1}} \norm{\nabla F(\bar{\x}^t)}^2\\
\nonumber
& \qquad - \frac{s\eta}{2} \norm{\frac{1}{m}\nabla \bm{F}^{(t)}\bm{1}}^2 + \kappa^2 \eta^2 s^2 L^2 (3\xi^2 +\sigma^2)\\
& \qquad + \pth{\frac{s\eta L^2 }{2m} + \kappa^2 \eta^2 s^2 L^2 \frac{3L^2}{m}} \sum_{i=1}^m \norm{\bar{\x}^{t} - \x_i^t}^2. 
\end{align}
\paragraph{Bounding $\expect{\norm{\frac{1}{m} \bm{G}^{(t)} \bm{1}}^2 \mid \calF^{t}}$}

So, we have 
\begin{align*}
\norm{\frac{1}{m} \bm{G}^{(t)} \bm{1}}^2 & =  \norm{\frac{1}{m} \sum_{i=1}^m \sum_{k=0}^{s-1} \nabla \ell_i^{(t,k)}}^2  \\
& =  \norm{\frac{1}{m} \sum_{i=1}^m \sum_{k=0}^{s-1} \pth{\nabla \ell_i^{(t,k)} - \nabla \ell_i^{(t,0)} + \nabla \ell_i^{(t,0)}}}^2 \\
& \le 2 \underbrace{\norm{\frac{1}{m} \sum_{i=1}^m \sum_{k=0}^{s-1} \pth{\nabla \ell_i^{(t,k)} - \nabla \ell_i^{(t,0)}}}^2}_{(\rmC)} + 2 \underbrace{\norm{\frac{s}{m} \sum_{i=1}^m   \nabla \ell_i^{(t,0)}}^2}_{(\rmD)}.  
\end{align*}
\noindent For term $(\rmC)$, %
by Lemma \ref{lemma: local step perturbation}, we have
\begin{align*}
\norm{\frac{1}{m} \sum_{i=1}^m \sum_{k=0}^{s-1} \pth{\nabla \ell_i^{(t,k)} - \nabla \ell_i^{(t,0)}}} 
& \le \frac{1}{m} \sum_{i=1}^m \norm{\sum_{k=0}^{s-1} \pth{\nabla \ell_i^{(t,k)} - \nabla \ell_i^{(t,0)}}} \\
&  \le  \frac{\kappa \eta s^2L}{2m} \sum_{i=1}^m \norm{\nabla \ell_i^{(t,0)}}. 
\end{align*}
Thus, we get 
\begin{align*}
\norm{\frac{1}{m} \sum_{i=1}^m \sum_{k=0}^{s-1} \pth{\nabla \ell_i^{(t,k)} - \nabla \ell_i^{(t,0)}}}^2  
&\le \frac{\kappa^2 \eta^2 s^4 L^2}{4m} \sum_{i=1}^m \norm{\nabla \ell_i^{(t,0)}}^2 \\
& \le \frac{\kappa^2 \eta^2 s^4 L^2}{2m} \pth{\sum_{i=1}^m \norm{\nabla \ell_i^{(t,0)} - \nabla F_i(\x_i^t)}^2 + \sum_{i=1}^m \norm{\nabla F_i(\x_i^t)}^2}. 
\end{align*}
By Assumption \ref{ass: bounded variance client-wise}, we obtain 
\begin{align*}
\expect{\norm{\frac{1}{m} \sum_{i=1}^m \sum_{k=0}^{s-1} \pth{\nabla \ell_i^{(t,k)} - \nabla \ell_i^{(t,0)}}}^2  \mid \calF^{t}}   
& \le  \frac{\kappa^2 \eta^2 s^4 L^2\sigma^2}{2} + \frac{\kappa^2 \eta^2 s^4 L^2}{2m} \sum_{i=1}^m \norm{\nabla F_i(\x_i^t)}^2. 
\end{align*}
\noindent For term $(\rmD)$,  by Assumption \ref{ass: bounded variance client-wise}, we have 
\begin{align*}
\expect{\frac{s^2}{m^2} \norm{\sum_{i=1}^m   \nabla \ell_i^{(t,0)}}^2\mid \calF^{t}}  
& \le  \expect{ \frac{2s^2}{m^2}\norm{\sum_{i=1}^m  \nabla \ell_i^{(t,0)} - \nabla F_i(\x_i^t)}^2 \mid \calF^{t}}+ \frac{2s^2}{m^2}\norm{\sum_{i=1}^m  \nabla F_i(\x_i^t)}^2\\
& \le \frac{2s^2 \sigma^2}{m}+ \frac{2s^2}{m^2}\norm{\sum_{i=1}^m  \nabla F_i(\x_i^t)}^2.
\end{align*}
Combing the above upper bounds of terms (C) and (D), we get 
\begin{align*}
\expect{\norm{\frac{1}{m} \bm{G}^{(t)} \bm{1}}^2 \mid \calF^{t}}  
& \le 2\qth{\frac{2s^2}{m} \sum_{i=1}^m \norm{\nabla F_i(\x_i^t)}^2 + \frac{\kappa^2 \eta^2 s^4 L^2}{2m} \sum_{i=1}^m \norm{\nabla F_i(\x_i^t)}^2  + s^2\sigma^2\pth{\frac{2}{m} + \frac{\kappa^2 \eta^2 s^2 L^2}{2}}} \\
& = s^2\pth{4 +\kappa^2 \eta^2 s^2 L^2 } \frac{1}{m}\sum_{i=1}^m \norm{\nabla F_i(\x_i^t)}^2 + s^2\sigma^2\pth{\frac{4}{m} + \kappa^2 \eta^2 s^2 L^2}.  
\end{align*}
Applying Proposition \ref{proposition: average gradient to global gradient}, we get 
\begin{align}
\label{eq: bound 2}
\nonumber
\expect{\norm{\frac{1}{m} \bm{G}^{(t)} \bm{1}}^2 \mid \calF^{t}} 
& \le 6 s^2L^2\pth{2 +\frac{\kappa^2 \eta^2 s^2 L^2}{2}} \frac{1}{m} \sum_{i=1}^m \norm{\x_i^t - \bar{\x}^t}^2 \\
\nonumber
& ~ +6 s^2\pth{\beta^2 + 1}\pth{2 +\frac{\kappa^2 \eta^2 s^2 L^2}{2}}\norm{\nabla F(\bar{\x}^t)}^2 \\
& ~ + 6 s^2\xi^2 \pth{2 +\frac{\kappa^2 \eta^2 s^2 L^2}{2}} + 2 s^2\sigma^2\pth{\frac{2}{m} + \frac{\kappa^2 \eta^2 s^2 L^2}{2}}. 
\end{align}
\paragraph{Putting them together.}
With Eq.\eqref{eq: bound 1} and \eqref{eq: bound 2}, we have 
\begin{align*}
\expect{F(\bar{\x}^{t+1})  -  F(\bar{\x}^{t}) \mid \calF^{t}}  
&\le \expect{\iprod{\nabla F(\bar{\x}^{t})}{- \frac{\eta}{m} \bm{G}^{(t)} \bm{1}} \mid \calF^{t}} +  \frac{L\eta^2}{2}\expect{\norm{\frac{1}{m} \bm{G}^{(t)} \bm{1}}^2 \mid \calF^{t}}\\
& \le - \qth{\frac{\eta s}{2} - \frac{\eta^2 s^2}{2} - 3\kappa^2 \eta^2 s^2 L^2\pth{\beta^2 + 1}} \norm{\nabla F(\bar{\x}^t)}^2\\
\nonumber
& \qquad - \frac{\eta s}{2} \norm{\frac{1}{m}\nabla \bm{F}^{(t)}\bm{1}}^2 + \kappa^2 \eta^2 s^2 L^2 (3\xi^2 +\sigma^2)\\
& \qquad + \pth{\frac{s\eta L^2 }{2m} + \kappa^2 \eta^2 s^2 \frac{3L^4}{m}} \sum_{i=1}^m \norm{\bar{\x}^{t} - \x_i^t}^2 \\
& \qquad + \frac{L\eta^2}{2}6s^2L^2\pth{2 +\frac{\kappa^2 L^2}{2}} \frac{1}{m} \sum_{i=1}^m \norm{\x_i^t - \bar{\x}^t}^2 \\
\nonumber
& \qquad +\frac{L\eta^2}{2}6 s^2\pth{\beta^2 + 1}\pth{2 +\frac{\kappa^2 L^2}{2}}\norm{\nabla F(\bar{\x}^t)}^2 \\
&\qquad + \frac{L\eta^2}{2}6 s^2\xi^2 \pth{2 +\frac{\kappa^2  L^2}{2}} + \frac{L\eta^2}{2} 2s^2\sigma^2\pth{\frac{2}{m} + \frac{\kappa^2 L^2}{2}}. 
\end{align*}
We can choose $\eta\le \frac{1}{2s}$ so that
\begin{align*}
\expect{F(\bar{\x}^{t+1})  -  F(\bar{\x}^{t}) \mid \calF^{t}}  
&\le \expect{\iprod{\nabla F(\bar{\x}^{t})}{- \frac{\eta}{m} \bm{G}^{(t)} \bm{1}} \mid \calF^{t}} +  \frac{L\eta^2}{2}\expect{\norm{\frac{1}{m} \bm{G}^{(t)} \bm{1}}^2 \mid \calF^{t}}\\
& \le - \sth{\frac{s\eta}{4} - 3 \eta^2 s^2 \pth{\beta^2 + 1}\qth{\kappa^2 L^2 + 2L \pth{1 + \frac{\kappa^2 L^2}{4}}}} \norm{\nabla F(\bar{\x}^t)}^2\\
\nonumber
& \qquad  +  3\xi^2 \eta^2 s^2 \qth{\kappa^2 L^2 + 2L \pth{1 + \frac{\kappa^2 L^2}{4}}} \\
& \qquad + \sigma^2 \eta^2 s^2 \qth{\kappa^2 L^2 + 2L \pth{\frac{1}{m} + \frac{\kappa^2 L^2}{4}}}\\
& \qquad +  \sth{\eta s L^2 + 3\eta^2 s^2 L^2\qth{\kappa^2 L^2 + 2L \pth{1 + \frac{\kappa^2 L^2}{4}}}} \frac{1}{m} \sum_{i=1}^m \norm{\x_i^t - \bar{\x}^t}^2.
\end{align*}
\end{proof}
\begin{proof}[\bf Proof of Lemma \ref{lemma: consensus}] 
Our proof shares the same outline as that in \cite{wang2022matcha} yet with non-trivial adaptation to account for multiple local updates and the fact the stochastic gradients at a client  within each round are {\em not independent}. Particularly, $\rmT_1$ in Eq.\,\eqref{eq: conseneus iterative error} does not exist in \cite{wang2022matcha}.  

We have the following relations:
\begin{align*}
    \bm{X}^{(t)} \pth{\identity - \allones} 
    &= (\bm{X}^{(t-1)} - \eta \bm{G}^{(t-1)}) W^{(t-1)} \pth{\identity - \allones}\\
    &= - \eta \sum_{q=0}^{t-1} \bm{G}^{(q)} \pth{\Pi_{\ell=q}^{t-1} W^{(\ell)} - \allones},
\end{align*}
where the last follows from the fact that all clients are initiated at the same weights.
It follows that
\begin{align}
\label{eq: conseneus iterative error}
\nonumber
\fnorm{\bm{X}^{(t)} \pth{\identity - \allones}}^2   
& \le 3\eta^2 \underbrace{\fnorm{
    \sum_{q=0}^{t-1} 
    \pth{
    \bm{G}^{\pth{q}} - \bm{G}_0^{\pth{q}}
    }
    \pth{\Pi_{\ell=q}^{t-1} W^{\pth{\ell}} - \allones}
    }^2}_{\rmT_1}\\
    \nonumber
    &~+ 3\eta^2 \underbrace{\fnorm{\sum_{q=0}^{t-1}\pth{\bm{G}_0^{\pth{q}} - s\nabla \bm{F}^{\pth{q}}}
    \pth{\prod_{\ell=q}^{t-1} W^{\pth{\ell}} - \allones}}^2}_{\rmT_2}\\
    &~+ 3\eta^2 s^2 \underbrace{\fnorm{\sum_{q=0}^{t-1}\nabla \bm{F}^{\pth{q}}
    \pth{\Pi_{\ell=q}^{t-1} W^{\pth{\ell}} - \allones}}^2}_{\rmT_3}. 
\end{align}

\paragraph{Bounding $\expect{\rmT_1}.$}
\begin{align}
\label{eq: conseneus iterative error 1}
\nonumber
    \expect{\rmT_1 } &= \sum_{q=0}^{t-1} \expect{\fnorm{ 
    \pth{
    \bm{G}^{\pth{q}} - \bm{G}_0^{\pth{q}}
    }
    \pth{\prod_{\ell=q}^{t-1} W^{\pth{\ell}} - \allones}
    }^2 } \\
    \nonumber
     &\qquad + \sum_{q=0}^{t-1} \sum_{p=0, p\neq q}^{t-1} \expect{\iprod{\pth{\bm{G}^{\pth{p}} - \bm{G}_0^{\pth{p}}} \pth{\Pi_{\ell=p}^{t-1}W^{\pth{\ell}} - \allones}}{\pth{
    \bm{G}^{\pth{q}} - \bm{G}_0^{\pth{q}}
    }
    \pth{\Pi_{\ell=q}^{t-1} W^{\pth{\ell}} - \allones}
    } }\\
    \nonumber
    &\overset{(a)}{\le} \sum_{q=0}^{t-1} \rho^{t-q}\expect{\fnorm{\bm{G}^{\pth{q}} - \bm{G}_0^{\pth{q}}}^2 } \\
    \nonumber
    &\qquad  + \sum_{q=0}^{t-1} \sum_{p=0, p\neq q}^{t-1}\expect{ 
    \fnorm{\pth{
    \bm{G}^{\pth{p}} - \bm{G}_0^{\pth{p}}
    }
    (\prod_{\ell=p}^{t-1} W^{\pth{\ell}} - \allones)}
    \fnorm{
    \pth{
    \bm{G}^{\pth{q}} - \bm{G}_0^{\pth{q}}
    }
    \pth{\Pi_{\ell=q}^{t-1} W^{\pth{\ell}} - \allones}
    } }\\
    \nonumber
    &\le \sum_{q=0}^{t-1} \rho^{t-q}\expect{\fnorm{\bm{G}^{\pth{q}} - \bm{G}_0^{\pth{q}}}^2 } \\
    & \qquad     + \sum_{q=0}^{t-1} \sum_{p=0, p\neq q}^{t-1} \expect{
    {
    \frac{\rho^{t-p}}{2\epsilon}
    \fnorm{\pth{
    \bm{G}^{\pth{p}} - \bm{G}_0^{\pth{p}}
    }}^2 
    +
    \frac{\epsilon \rho^{t-q}}{2}
    \fnorm{
    \pth{
    \bm{G}^{\pth{q}} - \bm{G}_0^{\pth{q}}
    }}^2 } }, 
\end{align}
where inequality (a) follows from Lemma \ref{lemma: average mixing}, and Cauchy-Schwarz inequality. 
Next, we bound the second term, choose $\epsilon = \rho^{\frac{q-p}{2}},$
\begin{align*}
    &\sum_{q=0}^{t-1} \sum_{p=0, p\neq q}^{t-1}  \frac{\sqrt{\rho}^{2t-p-q}}{2}\expect{
    {
    \fnorm{\pth{
    \bm{G}^{\pth{p}} - \bm{G}_0^{\pth{p}}
    }}^2 
    +
    \fnorm{
    \pth{
    \bm{G}^{\pth{q}} - \bm{G}_0^{\pth{q}}
    }}^2 } } \\ 
    & \le \sum_{q=0}^{t-1} \sum_{p=0}^{t-1} \frac{\sqrt{\rho}^{2t-p-q}}{2}\expect{
    {
    \fnorm{\pth{
    \bm{G}^{\pth{p}} - \bm{G}_0^{\pth{p}}
    }}^2 
    +
    \fnorm{
    \pth{
    \bm{G}^{\pth{q}} - \bm{G}_0^{\pth{q}}
    }}^2 }}\\
    & = \sum_{p=0}^{t-1}  
    \frac{\sqrt{\rho}^{t-p}}{2}\expect{
    \fnorm{\pth{
    \bm{G}^{\pth{p}} - \bm{G}_0^{\pth{p}}
    }}^2 }
    \sum_{q=0}^{t-1}\sqrt{\rho}^{t-q}
    +
    \sum_{q=0}^{t-1}  \frac{\sqrt{\rho}^{t-p}}{2}\expect{
    \fnorm{
    \pth{
    \bm{G}^{\pth{q}} - \bm{G}_0^{\pth{q}}
    }}^2 }
    \sum_{p=0}^{t-1}
    \sqrt{\rho}^{t-p}\\ 
    & = \frac{\sqrt{\rho} - \sqrt{\rho}^{t+1}}{1-\sqrt{\rho}}\sum_{q=0}^{t-1}  \sqrt{\rho}^{t-q}\expect{
    \fnorm{
    \pth{
    \bm{G}^{\pth{q}} - \bm{G}_0^{\pth{q}}
    }}^2 }.
\end{align*}
Plugging the above bound back in Eq.\eqref{eq: conseneus iterative error 1}, we get
\begin{align*}
\expect{\rmT_1} 
    & \le  \sum_{q=0}^{t-1} 
    \qth{\sqrt{\rho}^{t-q} + \frac{\sqrt{\rho} - \sqrt{\rho}^{t+1}}{1-\sqrt{\rho}}}
    \sqrt{\rho}^{t-q}\expect{\fnorm{\bm{G}^{\pth{q}} - \bm{G}_0^{\pth{q}}}^2 }\\
    & = \sum_{q=0}^{t-1} 
    \qth{\frac{\sqrt{\rho} + \sqrt{\rho}^{t+1} \pth{\frac{1 - \sqrt{\rho}}{\sqrt{\rho}^{q+1}}-1} }{1-\sqrt{\rho}}}
    \sqrt{\rho}^{t-q}\expect{\fnorm{\bm{G}^{\pth{q}} - \bm{G}_0^{\pth{q}}}^2  }\\
    & \le \sum_{q=0}^{t-1} 
    \qth{\frac{\sqrt{\rho} + \sqrt{\rho} \pth{1 - \sqrt{\rho} - \sqrt{\rho}^t} }{1-\sqrt{\rho}}}
    \sqrt{\rho}^{t-q}\expect{\fnorm{\bm{G}^{\pth{q}} - \bm{G}_0^{\pth{q}}}^2 }\\
    & \le \frac{2\sqrt{\rho}}{1-\sqrt{\rho}}\sum_{q=0}^{t-1} 
    \sqrt{\rho}^{t-q}\expect{\fnorm{\bm{G}^{\pth{q}} - \bm{G}_0^{\pth{q}}}^2 }.
\end{align*}
It remains to bound $\expect{\fnorm{\bm{G}^{\pth{q}} - \bm{G}_0^{\pth{q}}}^2 },$
\begin{align*}
    \expect{\fnorm{\bm{G}^{\pth{q}} - \bm{G}_0^{\pth{q}}}^2 } 
    & \overset{(a)}{\le} \kappa^2 \eta^2 \binom{s}{2}^2 L^2 \expect{\fnorm{\bm{G}_0^{\pth{q}} - s\nabla \bm{F}^{\pth{q}} + s\nabla \bm{F}^{\pth{q}}  }^2 }\\
    & \le 2 \kappa^2 \eta^2 \binom{s}{2}^2 L^2 \expect{\fnorm{\bm{G}_0^{\pth{q}} - s\nabla \bm{F}^{\pth{q}}}^2 } 
        + 2 \kappa^2 s^2 \eta^2 \binom{s}{2}^2 L^2 \expect{\fnorm{\nabla \bm{F}^{\pth{q}}}^2 }\\
    & \le 2 \kappa^2 \eta^2 \binom{s}{2}^2 L^2 m \sigma^2 
        + 2 \kappa^2 s^2 \eta^2 \binom{s}{2}^2 L^2 \expect{\fnorm{\nabla \bm{F}^{\pth{q}}}^2 }, 
\end{align*}
where inequality (a) follows from Lemma \ref{lemma: local step perturbation}.  
Thus,
\begin{align*}
    \expect{\rmT_1 } 
    & \le \frac{2\sqrt{\rho}}{1-\sqrt{\rho}}\sum_{q=0}^{t-1} 
    \sqrt{\rho}^{t-q}\expect{\fnorm{\bm{G}^{\pth{q}} - \bm{G}_0^{\pth{q}}}^2 }\\
    &\le \frac{2\sqrt{\rho}}{1-\sqrt{\rho}}\sum_{q=0}^{t-1} 
    \sqrt{\rho}^{t-q} \qth{
      2 \kappa^2 \eta^2 \binom{s}{2}^2 L^2 m \sigma^2 
    + 2 \kappa^2 s^2 \eta^2 \binom{s}{2}^2 L^2 \expect{\fnorm{\nabla \bm{F}^{\pth{q}}}^2 }
    }\\
    & \le \frac{4\kappa^2 \eta^2 \binom{s}{2}^2 L^2 m \sigma^2 \rho}{\pth{1-\sqrt{\rho}}^2}
        + \frac{4\kappa^2 s^2 \eta^2 \binom{s}{2}^2 L^2\sqrt{\rho}}{1-\sqrt{\rho}}
    \sum_{q=0}^{t-1} \sqrt{\rho}^{t-q}
    \expect{\fnorm{\nabla \bm{F}^{\pth{q}}}^2}
\end{align*}
\paragraph{Bounding $\expect{\rmT_2 }.$}
\begin{align*}
    \expect{\rmT_2 } & = \expect{\fnorm{
    \sum_{q=0}^{t-1} 
    \pth{\bm{G}_0^{\pth{q}} - s\nabla \bm{F}^{\pth{q}}}
    \pth{\Pi_{\ell=q}^{t-1} W^{\pth{\ell}} - \allones}
    }^2 }\\
    &\le \sum_{q=0}^{t-1}\rho^{t-q}\expect{
    \fnorm{
    \pth{\bm{G}_0^{\pth{q}} - s\nabla \bm{F}^{\pth{q}}}
    }^2}\\
    &\le \frac{\rho m s^2 \sigma^2}{1-\rho} .
\end{align*}
\paragraph{Bounding $\expect{\rmT_3}.$}
Use a similar trick as in bounding $\expect{\rmT_1},$ and we get
\begin{align*}
\expect{\rmT_3} &= \expect{\fnorm{\sum_{q=0}^{t-1}\nabla \bm{F}^{\pth{q}}\pth{\Pi_{\ell=q}^{t-1} W^{\pth{\ell}} - \allones}}^2}  \\
& \le \frac{2\sqrt{\rho}}{1-\sqrt{\rho}}\sum_{q=0}^{t-1} \sqrt{\rho}^{t-q}\expect{\fnorm{\nabla \bm{F}^{\pth{q}}}^2 }.
\end{align*}
For the last term, we have
\begin{align*}
&\frac{1}{mT}\sum_{t=0}^{T-1}\sum_{q=0}^{t-1} 
\sqrt{\rho}^{t-q}\expect{\fnorm{\nabla \bm{F}^{\pth{q}}}^2} 
= \frac{1}{mT}\sum_{k=0}^{T-1}\expect{\fnorm{\nabla \bm{F}^{\pth{t}}}^2} \sum_{q=1}^{T-1-t}\sqrt{\rho}^{q}\\
&\le \frac{\sqrt{\rho}}{mT\pth{1-\sqrt{\rho}}}\sum_{t=0}^{T-1}\expect{\fnorm{\nabla \bm{F}^{\pth{t}}}^2}.
\end{align*}
\paragraph{Putting them together.}
\begin{align*}
    &\frac{1}{mT}\sum_{t=0}^{T-1}\expect{\fnorm{\bm{X}^{\pth{t}} \pth{\identity - \allones}}^2}\\
    &\le \sigma^2 \rho \qth{\frac{12\kappa^2 \eta^4 \binom{s}{2}^2 L^2 }{\pth{1-\sqrt{\rho}}^2}
    + {\frac{3\eta^2 s^2}{1-\rho} }
    }+
    \qth{
    2\kappa^2 \eta^2 \binom{s}{2}^2 L^2 + 1
    }\frac{6 \eta^2 s^2\sqrt{\rho}}{1-\sqrt{\rho}}\frac{1}{mT}\sum_{t=0}^{T-1}\sum_{q=0}^{t-1} 
    \sqrt{\rho}^{t-q}\expect{\fnorm{\nabla \bm{F}^{\pth{q}}}^2 }\\
    &\le \sigma^2 \rho \qth{\frac{12\kappa^2 \eta^4 \binom{s}{2}^2 L^2 }{\pth{1-\sqrt{\rho}}^2}
    + {\frac{3\eta^2 s^2}{1-\rho} }
    }+
    \qth{
    2\kappa^2 \eta^2 \binom{s}{2}^2 L^2 + 1
    }\frac{6 \eta^2 s^2 \rho}{mT\pth{1-\sqrt{\rho}}^2}\sum_{t=0}^{T-1}\expect{\fnorm{\nabla \bm{F}^{\pth{t}}}^2}.
\end{align*}
We know that
\[
2\kappa^2 \eta^2 \binom{s}{2}^2 L^2 + 1 \le \frac{L^2 \kappa^2 s^2  \eta^2}{2} s^2 + 1 \le s^2 + 1 \le 2 s^2.
\]
Put all the parts together, we get
\begin{align*}
    &\frac{1}{mT}\sum_{k=0}^{T-1}\expect{\fnorm{\bm{X}^{\pth{t}} \pth{\identity - \allones}}^2}
    \le 6 s^2 \eta^2 \sigma^2 \rho \qth{\frac{2}{\pth{1-\sqrt{\rho}}^2}
    + {\frac{1}{1-\rho} }}
    +\frac{72 \xi^2 \eta^2 s^4 \rho}{\pth{1-\sqrt{\rho}}^2}
    +\frac{72 \pth{\beta^2 + 1} \eta^2 s^4 \rho}{\pth{1-\sqrt{\rho}}^2}\frac{1}{T}\sum_{t=0}^{T-1}\expect{\norm{\nabla F(\bar{\x}^t)}^2}, 
\end{align*}
which follows from the step size $\eta$
\[
    \frac{36 L^2 \eta^2 s^4 \rho}{\pth{1-\sqrt{\rho}}^2}\le \frac{1}{2}.
\]
\end{proof}
\begin{proof}[\bf Proof of Theorem \ref{thm: main}]
By taking an extra expectation over the remaining randomness and telescoping sum, we get
\begin{small}
\begin{align*}
&\frac{F^\star - F(\bar{\x}^0)}{T}\\
& \le - \qth{\frac{s\eta}{4} - 3\eta^2s^2\pth{\beta^2 + 1}\qth{\kappa^2 L^2 + 2 L \pth{1 + \frac{\kappa^2 L^2}{4}}}}\frac{1}{T}\sum_{t=0}^{T-1}\expect{\norm{\nabla F(\bar{\x}^t)}^2}\\
& \qquad +  3\xi^2 \eta^2 s^2 \qth{\kappa^2 L^2 + 2 L \pth{1 + \frac{\kappa^2 L^2}{4}}}\\
& \qquad +  \sigma^2 \eta^2 s^2 \qth{\kappa^2 L^2 + 2 L \pth{\frac{1}{m} + \frac{\kappa^2 L^2}{4}}}\\
& \qquad + \sth{s \eta L^2 + 3\eta^2 s^2 L^2\qth{\kappa^2 L^2 + 2 L \pth{1 + \frac{\kappa^2 L^2}{4}}}} 
    \sth{6 s^2 \eta^2 \sigma^2 \qth{\frac{2\rho }{\pth{1-\sqrt{\rho}}^2}
    + {\frac{\rho }{1-\rho} }}
    +\frac{72 \xi^2 \eta^2 s^4 \rho}{\pth{1-\sqrt{\rho}}^2}}\\
&\qquad + \sth{s \eta L^2 + 3\eta^2 s^2 L^2\qth{\kappa^2 L^2+ 2 L \pth{1 + \frac{\kappa^2 L^2}{4}}}}\qth{\frac{72 \pth{\beta^2 + 1} \eta^2 s^4 \rho}{\pth{1-\sqrt{\rho}}^2}\frac{1}{T}\sum_{t=0}^{T-1}\expect{\norm{\nabla F(\bar{\x}^t)}^2}}\\
& = - s \eta \sth{\frac{1}{4} - 3\eta s\pth{\beta^2 + 1}\qth{\kappa^2 L^2 + 2 L \pth{1 + \frac{\kappa^2 L^2}{4}}}\qth{1 + \frac{72 \eta^2 L^2 s^4 \rho}{\pth{1-\sqrt{\rho}}^2}}-\frac{72 \pth{\beta^2 + 1} L^2 \eta^2 s^4 \rho}{\pth{1-\sqrt{\rho}}^2}}\frac{1}{T}\sum_{t=0}^{T-1}\expect{\norm{\nabla F(\bar{\x}^t)}^2}\\
&\qquad + \eta^2 s^2 
\sth{\qth{\kappa^2 L^2 + 2 L \pth{\frac{1}{m} + \frac{\kappa^2 L^2}{4}}}
+ 6 \sth{s \eta L^2 + 3\eta^2 s^2 L^2\qth{\kappa^2 L^2 + 2 L \pth{1 + \frac{\kappa^2 L^2}{4}}}} 
    \qth{\frac{2\rho }{\pth{1-\sqrt{\rho}}^2}
    + {\frac{\rho }{1-\rho} }}
}\sigma^2\\
&\qquad + 3 \eta^2 s^2 \sth{
\qth{\kappa^2 L^2 + 2 L \pth{1 + \frac{\kappa^2 L^2}{4}}}
+ \sth{s \eta L^2 + 3\eta^2 s^2 L^2\qth{\kappa^2 L^2 + 2 L \pth{1 + \frac{\kappa^2 L^2}{4}}}}\frac{24 s^2 \rho}{\pth{1-\sqrt{\rho}}^2}
} \xi^2
\end{align*}
\end{small}
What follows refines the choice of the step-size:
\begin{footnotesize}
\begin{align*}
    &\frac{1}{4} - 3\eta s\pth{\beta^2 + 1}\qth{\kappa^2 L^2 + 2 L \pth{1 + \frac{\kappa^2 L^2}{4}}}\qth{1 + \frac{72 \eta^2 L^2 s^4 \rho}{\pth{1-\sqrt{\rho}}^2}}-\frac{72 \pth{\beta^2 + 1} L^2 \eta^2 s^4 \rho}{\pth{1-\sqrt{\rho}}^2} \\
    &\overset{(\rma)}{\ge}\frac{1}{4} - 3\eta s\pth{\beta^2 + 1}\qth{\kappa^2 L^2 + 2 L \pth{1 + \frac{\kappa^2 L^2}{4}}}\qth{1 + \frac{144 s^2 \rho}{\kappa^2 \pth{1-\sqrt{\rho}}^2}}-\frac{144 \pth{\beta^2 + 1} L \eta s^3 \rho}{\kappa\pth{1-\sqrt{\rho}}^2}\overset{(\rmb)}{\ge} \frac{1}{8},
\end{align*}
\end{footnotesize}
where $(\rma)$ follows because $\kappa \eta s L \le \sqrt{2}$
, while $(\rmb)$ because 
\[
\eta \le \frac{1}{24\pth{\beta^2 + 1}\qth{\kappa^2 L^2 + 2 L \pth{1 + \frac{\kappa^2 L^2}{4}}}\qth{1 + \frac{144 s^2 \rho}{\kappa^2 \pth{1-\sqrt{\rho}}^2}}+\frac{1152 \pth{\beta^2 + 1} L s^2 \rho}{\kappa\pth{1-\sqrt{\rho}}^2}}.
\]

\begin{align*}
    &\qth{\kappa^2 L^2 + 2 L \pth{\frac{1}{m} + \frac{\kappa^2 L^2}{4}}}
+ 6 \qth{s \eta L^2 + 3\eta^2 s^2 L^2\pth{\kappa^2 L^2 + 2 L \pth{1 + \frac{\kappa^2 L^2}{4}}}} 
    \qth{\frac{2\rho }{\pth{1-\sqrt{\rho}}^2}
    + {\frac{\rho }{1-\rho} }} \\
   \le & \qth{\kappa^2 L^2 + 2 L \pth{\frac{1}{m} + \frac{\kappa^2 L^2}{4}}}
+ 6 \qth{s\eta L^2 + \frac{3\sqrt{2} s \eta  L}{\kappa}  \pth{\kappa^2 L^2 + 2 L \pth{1 + \frac{\kappa^2 L^2}{4}}}} 
    \qth{\frac{2\rho }{\pth{1-\sqrt{\rho}}^2}
    + {\frac{\rho }{1-\rho} }} 
\end{align*}
In addition, we need to ensure that $\eta \rho s^3\le 1$, with such an additional choice, we get 
\begin{align*}
&\qth{\kappa^2 L^2 + 2 L \pth{1 + \frac{\kappa^2 L^2}{4}}}
+ \sth{s \eta L^2 + 3\eta^2 s^2 L^2\qth{\kappa^2 L^2 + 2 L \pth{1 + \frac{\kappa^2 L^2}{4}}}}\frac{24 s^2 \rho}{\pth{1-\sqrt{\rho}}^2}\\
\le & \qth{\kappa^2 L^2 + 2 L \pth{1 + \frac{\kappa^2 L^2}{4}}}
+ \sth{1 + 3\eta s \qth{\kappa^2 L^2 + 2 L \pth{1 + \frac{\kappa^2 L^2}{4}}}}\frac{24 L^2}{\pth{1-\sqrt{\rho}}^2}.
\end{align*}
A little rearrangement, and applying the fact that
\[
1 - \rho = \pth{1 - \sqrt{\rho }}\pth{1 + \sqrt{\rho}} \ge \pth{1 - \sqrt{\rho }}^2, 
\]
 we arrive at
\begin{align*}
\frac{1}{T}\sum_{t=0}^{T-1}\expect{\norm{\nabla F(\bar{\x}^t)}^2} 
& \le \frac{8F(\bar{\x}^0) - 8 F^\star }{s \eta T} \\
&~+ 8 \eta s \kappa^2 L^2
\pth{1 + \frac{L}{2}} 
 \sigma^2\\ 
&~+ \frac{144  \rho \pth{\eta s }^2  }{\pth{1-\sqrt{\rho}}^2}   \sth{ L^2 + \frac{3\sqrt{2} L}{\kappa}  \qth{\kappa^2 L^2 + 2 L \pth{1 + \frac{\kappa^2 L^2}{4}}}}\sigma^2 \\
&~ +  \frac{16 \eta s L \sigma^2}{m}\\
&~+ 24 \eta s \sth{
\qth{\kappa^2 L^2 + 2 L \pth{1 + \frac{\kappa^2 L^2}{4}}}
+ \frac{24 L^2}{\pth{1-\sqrt{\rho}}^2}
} \xi^2\\
&~ + \frac{1728 L^2 \pth{\eta s }^2}{\pth{1-\sqrt{\rho}}^2}\qth{\kappa^2 L^2 + 2 L \pth{1 + \frac{\kappa^2 L^2}{4}}} \xi^2.
\end{align*}

Choose the step size to be $\eta = \sqrt{\frac{m}{sT}}.$
When $T$ is sufficiently large such that 
\begin{align*}
\eta \le \min \left\{\frac{1}{24\pth{\beta^2 + 1}\qth{\kappa^2 L^2 + 2 L \pth{1 + \frac{\kappa^2 L^2}{4}}} \qth{1 + \frac{144 s^2 \rho}{\kappa^2 \pth{1-\sqrt{\rho}}^2}}+\frac{1152 \pth{\beta^2 + 1} L s^2 \rho}{\kappa\pth{1-\sqrt{\rho}}^2}}, \frac{1}{2s}, 
\frac{\sqrt{2}}{\kappa sL},\frac{1}{\rho s^3},
\frac{1-\sqrt{\rho}}{6\sqrt{2\rho}Ls^2}\right\},
\end{align*}
we have 
\begin{align*}
\frac{1}{T}\sum_{t=0}^{T-1}\expect{\norm{\nabla F(\bar{\x}^t)}^2} 
& \le O \left\{\frac{8F(\bar{\x}^0) - 8 F^\star }{\sqrt{msT}}\right. \\
&~+ 8  \kappa^2 L^2
\pth{1 + \frac{L}{2}} 
 \sigma^2\sqrt{\frac{ms}{T}}\\ 
&~+ \frac{144  \rho  }{\pth{1-\sqrt{\rho}}^2}   \sth{ L^2 + \frac{3\sqrt{2} L}{\kappa}  \qth{\kappa^2 L^2 + 2 L \pth{1 + \frac{\kappa^2 L^2}{4}}}}\sigma^2 \frac{ms}{T}\\
&~ +  16 L \sigma^2\sqrt{\frac{s}{mT}}\\
&~+ 24 \sth{
\qth{\kappa^2 L^2 + 2 L \pth{1 + \frac{\kappa^2 L^2}{4}}}
+ \frac{24 L^2}{\pth{1-\sqrt{\rho}}^2}
} \xi^2 \sqrt{\frac{ms}{T}}\\
&~ \left.+ \frac{1728 L^2 }{\pth{1-\sqrt{\rho}}^2}\qth{\kappa^2 L^2 + 2 L \pth{1 + \frac{\kappa^2 L^2}{4}}} \xi^2\frac{ms}{T}\right\}.
\end{align*}
\end{proof}
\end{document}